\documentclass{article}

\usepackage{microtype}
\usepackage{graphicx}
\usepackage{subcaption}
\usepackage{booktabs} 

\usepackage{hyperref}

\usepackage[accepted]{icml2026}

\usepackage{amsmath}
\usepackage{amssymb}
\usepackage{mathtools}
\usepackage{amsthm}

\usepackage[capitalize,noabbrev]{cleveref}

\theoremstyle{plain}
\newtheorem{theorem}{Theorem}
\newtheorem{proposition}[theorem]{Proposition}

\theoremstyle{definition}
\newtheorem{definition}[theorem]{Definition}

\theoremstyle{remark}
\newtheorem{remark}[theorem]{Remark}

\usepackage[disable,textsize=tiny]{todonotes}

\usepackage[utf8]{inputenc}
\usepackage[T1]{fontenc}
\usepackage{enumitem}   
\usepackage{float}      
\usepackage{multirow} 
\usepackage{makecell}

\usepackage[most]{tcolorbox}
\usepackage{tabularx}
\usepackage{arydshln}
\definecolor{Mycolor1}{HTML}{EF9A9A}
\definecolor{Mycolor2}{HTML}{90CAF9}
\definecolor{Mycolor3}{HTML}{AED581}
\definecolor{mygray}{gray}{0.9}

\icmltitlerunning{Measuring and Mitigating Post-Hoc Rationalization in Reverse Chain-of-Thought Generation}

\begin{document}
\twocolumn[
  \icmltitle{
  Measuring and Mitigating Post-Hoc Rationalization \\in Reverse Chain-of-Thought Generation}



  \icmlsetsymbol{intern}{\ensuremath{\dagger}}

    \begin{icmlauthorlist}
      \icmlauthor{Guangyue Peng}{pku,intern}
      \icmlauthor{Zongchao Chen}{boss}
      \icmlauthor{Wen Luo}{pku}
      \icmlauthor{Yuntao Wen}{uestc}
      \icmlauthor{Wei Li}{pku}
      \icmlauthor{Ruixiang Feng}{uestc}
      \icmlauthor{Ran Le}{boss}
      \icmlauthor{Chen Yang}{boss}
      \icmlauthor{Zhenwei An}{boss}
      \icmlauthor{Yang Song}{boss}
      \icmlauthor{Tao Zhang}{boss}
      \icmlauthor{Houfeng Wang}{pku}
    \end{icmlauthorlist}
    
    \icmlaffiliation{pku}{State Key Laboratory of Multimedia Information Processing, School of Computer Science, Peking University}
    \icmlaffiliation{uestc}{University of Electronic Science and Technology of China}
    \icmlaffiliation{boss}{Nanbeige Lab, BOSS Zhipin}

  \icmlcorrespondingauthor{Yang Song}{songyang@kanzhun.com}
  \icmlcorrespondingauthor{Houfeng Wang}{wanghf@pku.edu.cn}

  \icmlkeywords{Machine Learning, ICML}

  \vskip 0.3in
]




\printAffiliationsAndNotice{
  \textsuperscript{\ensuremath{\dagger}}Work done during internship at Nanbeige Lab.
}

\begin{abstract}

Reverse Chain-of-Thought Generation (RCG) synthesizes reasoning traces from query-answer pairs, but answer-visible generation can justify a pre-committed answer rather than derive it. This post-hoc rationalization creates a train-inference mismatch because student models are trained on answer-conditioned traces but must reason without answer access at inference time. We quantify this mismatch through lexical, trajectory, and probabilistic anchoring, measuring surface overlap, answer-conditioned generation dynamics, and answer recoverability from the trace, respectively. We find that semantic suppression, a seemingly intuitive mitigation, reduces lexical overlap but increases trajectory anchoring: avoiding the answer requires continually tracking it, thereby strengthening its influence on generation. We therefore propose Structural Skeleton-guided Reasoning (SSR), which replaces suppression with structural decoupling by first generating an abstract functional skeleton and then using it to guide the full reasoning trace. Anchoring analyses show that SSR reduces all three forms of answer dependence. Across in-domain and out-of-distribution benchmarks, its distilled variant, SSR-D, improves performance by up to 10 points over suppression baselines and better preserves out-of-distribution performance. Code is available at \url{https://github.com/viniferagy/SSR}.

\end{abstract}

\section{Introduction}

\begin{figure}[t]
\centering
\includegraphics[width=0.98\linewidth]{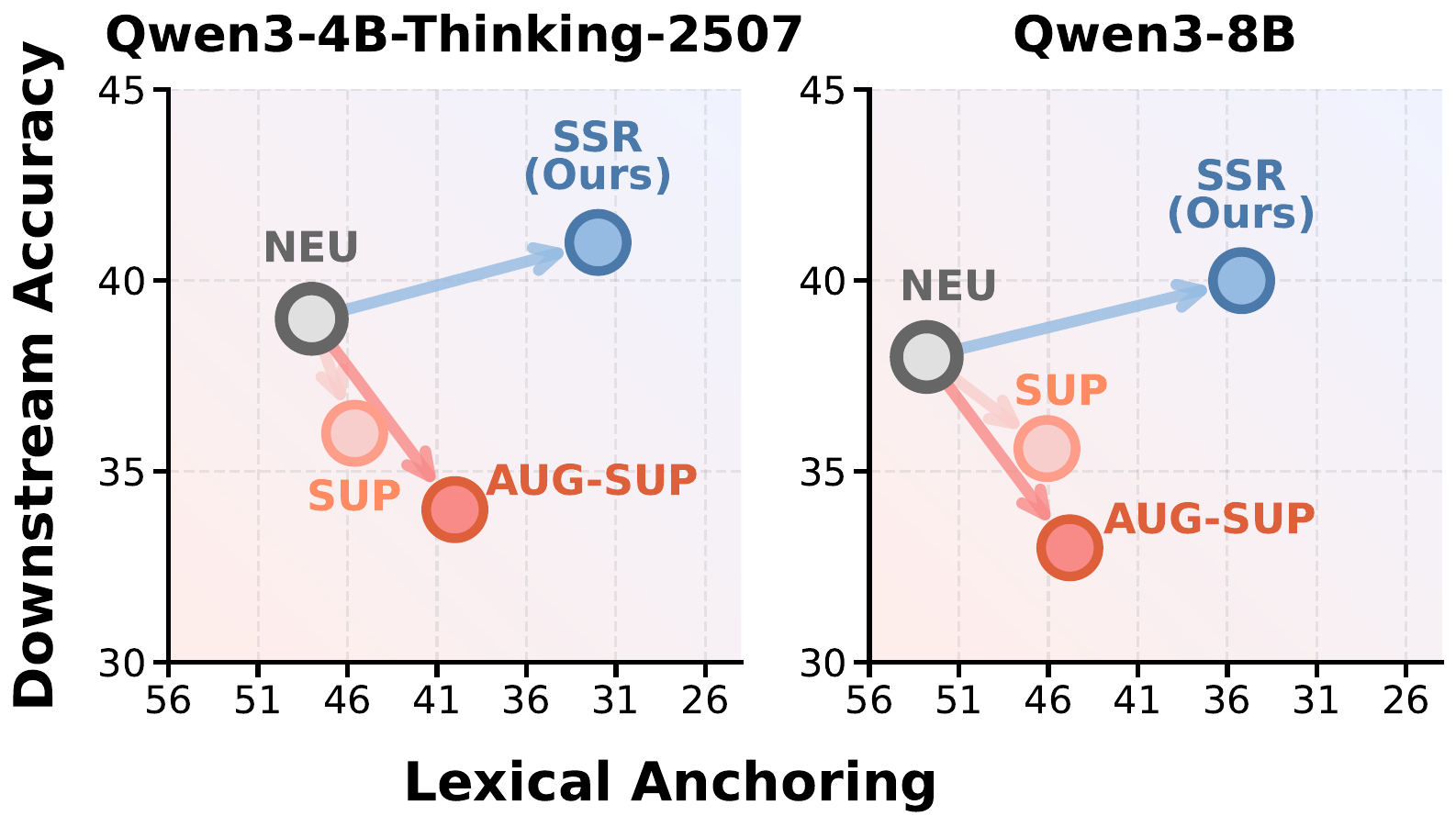}
\caption{Analysis of the relationship between \textit{Lexical Anchoring} and \textit{Downstream Accuracy}  (from \cref{tab:main_results}). The plot reveals that lexical anchoring is a poor indicator of model utility, as semantic suppression methods (SUP/AUG-SUP) tend to "deceive" this metric, obtaining smaller anchoring while simultaneously suffering from a drop in actual downstream performance. In contrast, our SSR method achieves improvements in both dimensions.}
\label{fig:intuition}
\end{figure}

The effectiveness of Large Language Models (LLMs) in complex reasoning tasks depends critically on the quality of intermediate reasoning traces \citep{wei2022chain, kojima2022large, chu2024navigate}. While expert-verified query-answer pairs are abundant \citep{cobbe2021training, hendrycks2021measuring}, the scarcity of corresponding step-by-step derivations creates a significant bottleneck for acquiring and transferring reasoning capabilities. Reverse Chain-of-Thought Generation (RCG) addresses this gap by synthesizing intermediate reasoning steps that logically bridge a query to a known answer \citep{bhagavatula2020abductive, zelikman2022star, li2025from}.

However, RCG is susceptible to post-hoc rationalization \citep{cox2025posthoc,jin2026mirage}. When the answer is visible during generation, models tend to rationalize backward from the conclusion rather than genuinely derive it \citep{turpin2023language,lanham2023measuring,lewis2025analysing}. Importantly, post-hoc rationalization does not necessarily degrade the surface quality or final accuracy of responses \citep{bentham2024chain}. Rather, it undermines the reliability and utility of the reasoning traces themselves \citep{agarwal2024faithfulness,paul2024making}. Because the model has committed to the response from the outset, this pre-determined response serves as a cognitive anchor that shapes the entire explanation \citep{bao2025how}. The resulting chain-of-thought becomes less logically self-contained. An observer presented only with the query, unaware of the anchored response, would find the reasoning less accessible and coherent \citep{madaan2023makes,arcuschin2025cot}. This anchoring effect weakens the utility of generated traces for explainability, faithfulness verification, reasoning distillation, and reliable measurement of out-of-distribution generalization \citep{chua2024bias,david2025temporal,sakana2025rlt}.

Crucially, post-hoc rationalization introduces a structural train-inference mismatch: reasoning traces are generated under $P(R \mid Q, A)$ with full answer visibility, but distillation aims to equip students who must reason under $P(R \mid Q)$ without answer access. The degree of PHR directly determines the divergence between these two distributions. Students trained on high-PHR traces learn a ``destination-aware'' reasoning style--linear, confident, and non-exploratory--that fails when the destination is unknown. This failure manifests most severely in out-of-distribution generalization, where pattern-matching to training examples cannot compensate for the mismatch.

To quantify and mitigate this mismatch, we propose a three-level hierarchy of anchoring metrics \citep{lanham2023measuring,paul2024making,bentham2024chain}. Beyond surface-level lexical anchoring, which simply measures token overlap between the trace and the anchored response, we examine deeper generation dynamics of the reverse chain-of-thought. We introduce trajectory anchoring to measure the per-token influence of answer visibility on the generation process, and probabilistic anchoring to measure the total information transmission from the reasoning trace to the anchored response.

We apply our measurement framework to evaluate mitigation strategies for post-hoc rationalization \citep{tanneru2024hardness}. The intuitive baseline is \textit{semantic suppression}: explicitly prompting models to ignore the given response or suppress indicators of pre-determination \citep{sakana2025rlt, wang2025reverse}. Although widely adopted as an intuitive remedy to successfully relieve the lexical anchoring, our analysis reveals that suppression fails on internal anchoring metrics. While it masks lexical overlaps, it increases trajectory anchoring--the generation process becomes more dependent on the answer the model is trying to suppress--while leaving total information mismatch largely unresolved, ultimately degrading the quality of reasoning traces and hurting downstream task performance (\cref{fig:intuition}). The model rationalizes more subtly, but more heavily. We connect this paradox with Ironic Process Theory \citep{wegner1994ironic} from cognitive psychology: instructing a model to ignore an answer forces it to actively monitor for exclusion, thereby deepening the specific dependence it aims to cut off (illustrated in \cref{fig:intro} (b)).

To break this cycle, we propose \textbf{Structural Skeleton-guided Reasoning (SSR)}. Rather than suppressing the answer, SSR decouples the structure of reasoning from its content. Generation proceeds in two phases: (1) synthesizing a "skeleton" of functional tags (e.g., \texttt{PLAN} $\rightarrow$ \texttt{INFR}) that extracts the coarse reasoning structure behind the response while discouraging direct encoding of specific response content, and (2) using this skeleton to guide generation of the full reasoning chain. By providing a content-neutral structural target, SSR reduces anchoring across the three-level hierarchy without relying on explicit answer-suppression instructions. Because prompted SSR can still be fragile, with models omitting tags or leaking answer details into skeletons, we further introduce \textbf{Distilled SSR (SSR-D)}. SSR-D first trains an SSR-format teacher from constructed skeleton-reasoning pairs, then uses this teacher to generate distilled SSR traces for the target model. The target model learns to generate a valid skeleton and then reconstruct the full reasoning from it, turning SSR from a prompting procedure into an internalized generation pattern.

Our contributions are:
\begin{enumerate}
\item \textbf{Anchoring Measurement:} We propose a three-level framework (lexical, trajectory, probabilistic) to quantify the train-inference mismatch induced by answer visibility in reverse chain-of-thought generation. Each level is answer-conditional, capturing progressively deeper forms of answer dependence.
\item \textbf{Mechanism Analysis:} We demonstrate that semantic suppression, the intuitive mitigation strategy for post-hoc rationalization, fails due to an "ironic process" that paradoxically strengthens anchoring.
\item \textbf{Methodology:} We introduce SSR to mitigate post-hoc rationalization by decoupling reasoning structure from anchored content, and further propose SSR-D as a distillation variant that trains target models on SSR traces generated by a fine-tuned SSR teacher for more reliable structural alignment and stronger downstream gains.
\end{enumerate}

\begin{figure*}[t]
    \centering
    \includegraphics[width=0.88\textwidth]{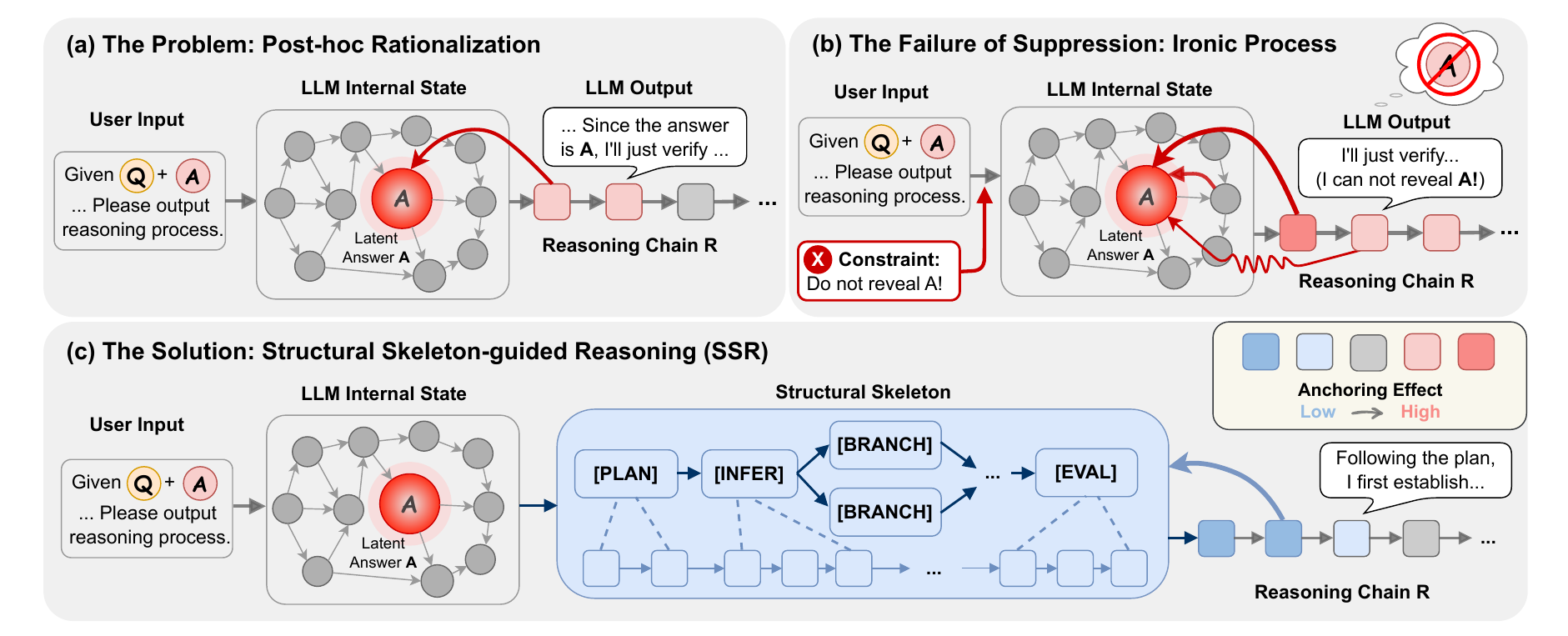}
    \caption{\textbf{Paradigms in Reverse Chain-of-Thought Generation.} The curved red and blue arrows indicate the anchoring effect of the pre-committed answer. (a) Post-hoc Rationalization: Visible answers cause shortcutting, resulting in \colorbox{Mycolor1!80}{rationalized} reasoning chains. (b) Suppression Failure: Negative constraints trigger "ironic monitoring" of the forbidden answer, paradoxically maintaining answer-steered generation trajectories and highly \colorbox{Mycolor1!80}{rationalized} chains. (c) Structural Skeleton-guided Reasoning (SSR): Decoupling content via an abstract skeleton redirects generation toward structural planning rather than answer monitoring, producing \colorbox{Mycolor2!80}{unanchored} chains driven by structural information rather than answer dependency.}
    \label{fig:intro}
\end{figure*}
\section{Anchoring Measurement}

The core source of train-inference mismatch in RCG is the dependence of $R$ on the answer $A$. At training time, $A$ is visible; at inference time, it is not. The more information about $A$ that is encoded in $R$--whether through surface tokens, generation-process dynamics, or latent statistical structure--the larger the gap between the training distribution $P(R \mid Q,A)$ and the inference distribution $P(R \mid Q)$. We quantify this dependence through three levels of anchoring metrics, ordered from observable surface features to latent information-theoretic properties.

We define \emph{Reverse Chain-of-Thought Generation} (RCG) as follows: given a query $Q$ and a pre-committed response $A$, generate a reasoning chain
\[
R = (r_1, r_2, \ldots, r_T)
\]
such that $R$ constitutes a coherent derivation from $Q$ to $A$. Unlike standard chain-of-thought prompting where the answer emerges from reasoning, RCG constructs explanatory traces for predetermined conclusions.

The central question in evaluating RCG is: \emph{To what extent does the visible answer $A$ create a train-inference mismatch, shaping the reasoning trace $R$ in ways that depend on information unavailable at inference time?}
We formalize this through a hierarchy of three metrics, ordered from observable surface features to latent information-theoretic properties.

\subsection{Lexical Anchoring ($\mathcal{A}_{\textup{lex}}$)}

\begin{definition}
Lexical anchoring captures the appearance of response-specific content within the reasoning chain. Because many answer words are already licensed by the query itself, we use a question-filtered IDF-weighted recall over answer content tokens:
\begin{equation}
\mathcal{C}_{Q}(A)=\mathcal{C}(A)\setminus\mathcal{C}(Q),
\end{equation}
\begin{equation}
\mathcal{A}_{\textup{lex}}
=
\frac{
\sum_{w \in \mathcal{C}_{Q}(A)}
\mathrm{idf}(w)\min(\mathrm{tf}_R(w), \mathrm{tf}_A(w))
}{
\sum_{w \in \mathcal{C}_{Q}(A)}
\mathrm{idf}(w)\mathrm{tf}_A(w)
}.
\end{equation}
Here $\mathcal{C}(X)$ denotes the multiset of non-stopword content tokens in text $X$, $\mathcal{C}_{Q}(A)$ removes answer content tokens that already appear in the query, $\mathrm{tf}_X(w)$ is the token frequency of $w$ in text $X$, and $\mathrm{idf}(w)$ is estimated from the response corpus. If the denominator is zero, the score is defined as zero.
\end{definition}

While $\mathcal{A}_{\textup{lex}}$ focuses on answer-specific content words rather than function-word overlap, it remains a surface-level indicator. It is therefore an insufficient proxy for evaluating the quality of generated reasoning. As in \cref{fig:intuition}, reducing $\mathcal{A}_{\textup{lex}}$ does not guarantee improved downstream performance, because models may simply learn to obscure the same underlying dependence.

\subsection{Trajectory Anchoring ($\mathcal{A}_{\textup{traj}}$)}

Beyond lexical artifacts, answer dependence manifests in the generation process itself. When the model can see the answer, its per-token predictions become more confident: the answer resolves uncertainty about what to write next. This confidence gap between answer-visible and answer-blind generation directly measures how much the generation \emph{trajectory} is steered by the answer.

\begin{definition}
For each token position $t$ in the reasoning trace $R$, let $H(p(\cdot \mid C))$ denote the entropy of the next-token predictive distribution under context $C$. Trajectory anchoring is the average entropy reduction attributable to answer visibility:
\begin{equation}
\mathcal{A}_{\textup{traj}} =
\frac{1}{|R|}
\sum_{t=1}^{|R|}
\left[
H\bigl(p(\cdot \mid Q, R_{[:t]})\bigr)
-
H\bigl(p(\cdot \mid Q, A, R_{[:t]})\bigr)
\right].
\end{equation}
\end{definition}

A high $\mathcal{A}_{\textup{traj}}$ indicates that knowing the answer substantially reduces the model's uncertainty at each generation step: the generation trajectory is strongly tethered to the answer. This metric captures train-inference mismatch at the \emph{process} level. During RCG, the model generates with answer access (low entropy, high confidence); during inference, the student generates without it (high entropy, lower confidence). The per-token gap between these two regimes is precisely $\mathcal{A}_{\textup{traj}}$.

Unlike lexical anchoring, which detects surface-level leakage after generation, trajectory anchoring captures the ongoing influence of the answer \emph{during} generation. A trace may contain no answer tokens ($\mathcal{A}_{\textup{lex}} \approx 0$) yet exhibit high trajectory anchoring if the model silently consults the answer at every step to guide its generation path.

\subsection{Probabilistic Anchoring ($\mathcal{A}_{\textup{prob}}$)}

\begin{definition}
Probabilistic anchoring quantifies the extent to which the reasoning trace reduces uncertainty about the pre-committed response. Let
\begin{equation}
\ell_\theta(A \mid C) = \frac{1}{|A|}\log_2 P_\theta(A \mid C)
\end{equation}
be the average log-probability of the response under context $C$. We compute the raw probabilistic anchoring score as the clipped fraction of baseline answer surprisal removed by the full reasoning trace:
\begin{equation}
\mathcal{A}_{\textup{prob}}
=
\operatorname{clip}_{[0,1]}\!\left(
\frac{\ell_\theta(A \mid Q,R)-\ell_\theta(A \mid Q)}
{-\ell_\theta(A \mid Q)}
\right).
\end{equation}
\end{definition}

A high $\mathcal{A}_{\textup{prob}}$ indicates that the reasoning trace $R$ explains a large fraction of the model's initial uncertainty about $A$. While some predictive capability is expected from valid reasoning, a disproportionately high score suggests that the trace essentially encodes the response directly, serving as a compressed transmission channel for $A$. This metric captures the deepest level of train-inference mismatch: information about the answer that persists in the trace even when surface tokens and generation-process signatures are successfully masked. Together with $\mathcal{A}_{\textup{lex}}$ and $\mathcal{A}_{\textup{traj}}$, it completes a hierarchy from surface overlap to process dependence to total information transmission. For compact presentation, scalar result tables use the same reference-normalized display scores as the behavioral-zone plots; normalization details are given in Appendix~\ref{app:display_score_normalization}.

\section{Methodology}
\label{sec:method}

\subsection{Baselines}

\paragraph{Neutral Prompting (NEU).}
Our baseline employs standard chain-of-thought generation where the model is given both the query $Q$ and response $A$, then asked to produce a reasoning trace connecting them without additional constraints. This represents the default RCG setting and establishes reference anchoring levels against which mitigation strategies are compared.\footnote{Detailed prompts are provided in Appendix~\ref{app:prompt}.}

\paragraph{Semantic Suppression (SUP).}
A natural mitigation strategy instructs the model to conceal the response during generation: \emph{``Reason step by step, but do not reveal the answer until the end.''} \citep{wang2025reverse} We additionally evaluate an intensified variant (\textbf{AUG-SUP}) with stronger suppression instructions threatening the model not to disclose the relevant information \citep{xu2025bullying}. While intuitively appealing, we hypothesize and empirically demonstrate that this approach fails to reduce internal anchoring. Suppression preserves or amplifies response information within the reasoning dynamics despite successfully masking it in surface text, challenging the assumption that lexical concealment equates to genuine derivation.

The main text focuses on NEU, SUP, and AUG-SUP as the core comparison set; additional baseline variants and stress-test diagnostics are discussed in Appendix~\ref{app:extended_mismatch_metrics}.

\subsection{Structural Skeleton-guided Reasoning (SSR)}
We propose SSR as an alternative that shifts from \emph{suppression} to \emph{separation}. Instead of forbidding access to the response, SSR borrows the insight from the meta-reasoning paradigm \citep{wang2023plan,ning2024skeleton} and introduces an intermediate structural representation that decouples reasoning topology from specific semantic content in the anchored response.

We define a \textbf{Structural Skeleton}
\[
S = \langle (f_1, c_1), \ldots, (f_n, c_n) \rangle
\]
as a sequence of abstract steps, where each step comprises a \emph{functional tag} $f_i$ from a closed set and a \emph{content summary} $c_i$ describing the step's intent without revealing values (e.g., ``calculate the ratio'' rather than ``calculate $0.5$''). This abstraction is designed to encourage invariance to the specific response while preserving the logical derivation.\footnote{Appendix~\ref{app:ssr_theory} provides theoretical properties for SSR.}

We report component ablations and SSR prompt/rendering variants in Appendix~\ref{app:ablation}, while keeping the main methodology centered on the final natural-language SSR design.

\newtcolorbox{tagbox}[1]{%
  enhanced,
  colback=Mycolor2!40,
  colframe=black,
  boxrule=0.9pt,
  arc=6pt,
  left=6pt,right=6pt,top=7pt,bottom=6pt,
  fonttitle=\bfseries\small,
  title={#1},
  attach boxed title to top left={xshift=8pt,yshift*=-\tcboxedtitleheight/2},
  boxed title style={
    colback=black,
    colframe=black,
    arc=4pt,
    boxrule=0pt,
    left=6pt,right=6pt,top=2pt,bottom=2pt,
  },
  coltitle=white,
}

\begin{tagbox}{Functional Tags}
\footnotesize
\setlength{\tabcolsep}{2.5pt}
\renewcommand{\arraystretch}{0.92}

\begin{tabular*}{\linewidth}{@{\extracolsep{\fill}}cccc@{}}
\texttt{[PLAN]} & \texttt{[RETR]} & \texttt{[INFR]} & \texttt{[EVAL]}\\[0.5pt]
Planning      & Retrieval      & Inference      & Evaluation\\[1.0pt]
\texttt{[SUMM]} & \texttt{[BTRK]} & \texttt{[RFLX]} & \texttt{[BRCH]}\\[0.5pt]
Summary       & Backtrack      & Reflection     & Branch\\
\end{tabular*}

\end{tagbox}

SSR operates via two-phase generation:
\begin{enumerate}[leftmargin=*, itemsep=0pt]
    \item \textbf{Skeleton Generation:} The model generates a skeleton conditioned on query and response:
    \[
    S \sim P(S \mid Q, A).
    \]
    \item \textbf{Reasoning Generation:} The model generates the full reasoning trace guided by the skeleton:
    \[
    R \sim P(R \mid S, Q, A).
    \]
\end{enumerate}

Although $A$ remains visible in both stages, the skeleton $S$ acts as an intermediate structural target that can reduce direct response encoding during reasoning generation. By providing a content-neutral target, SSR redirects computation toward structural organization.

\begin{figure}[t]
  \centering
  \includegraphics[width=0.76\linewidth]{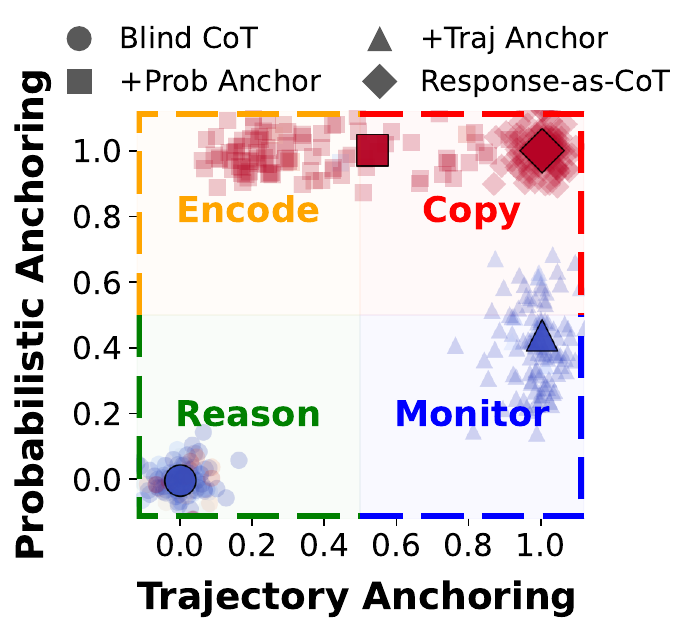}
  \caption{Controlled reference calibration for the behavioral-zone display. Blind CoT and Response-as-CoT use real model outputs as the low- and high-anchoring endpoints, while +Prob Anchor and +Traj Anchor are mechanically constructed control CoTs that separate probabilistic from trajectory anchoring.}
  \label{fig:behavioral_zones}
\end{figure}

\begin{figure*}[t]
  \centering
  \includegraphics[width=0.96\textwidth]{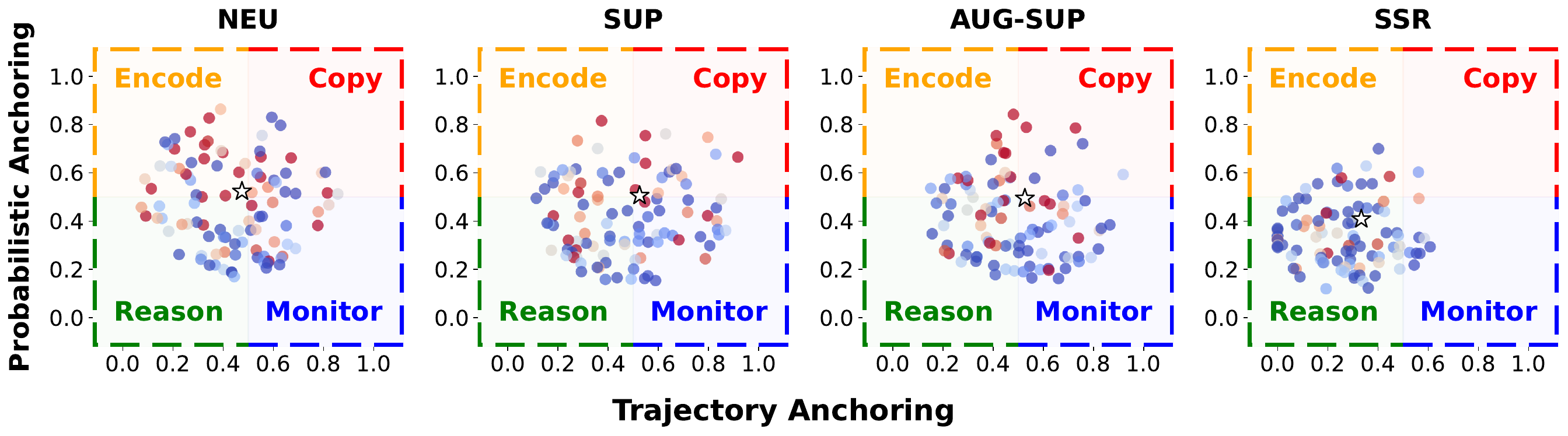}
  \caption{Mechanism diagnosis across generation strategies. Suppression (SUP, AUG-SUP) slightly reduces Reason-zone mass and increases Monitor/Copy behavior, while SSR concentrates more traces in the Reason zone and reduces pathological spread. Zone names follow the Trajectory--Probabilistic interpretation: Reason, Encode, Monitor, and Copy. Point color indicates lexical anchoring (\textcolor{blue}{blue}: low, \textcolor{red}{red}: high), and the white star in each panel marks the method-level centroid computed from mean trajectory and probabilistic anchoring.}
  \label{fig:mechanism_diagnosis}
\end{figure*}

\section{Experiments}

We evaluate strategies for mitigating post-hoc rationalization in reverse chain-of-thought generation. Our experimental framework measures the extent to which different prompting and generation approaches reduce anchoring effects while maintaining reasoning quality.

\subsection{Data Construction}
We sample 10{,}000 queries from LMArena \citep{chiang2024chatbot}, paired with reference responses generated via a Qwen3-Max self-improvement pipeline (details in Appendix~\ref{app:data_construction}). This setup simulates realistic reasoning distillation scenarios: high-quality responses emerge from multi-turn deliberation, but no gold-standard reasoning traces are available for supervision. Using Qwen3-4B-Thinking-2507 as the target model,\footnote{We validate observations on Qwen3-8B in Appendix~\ref{app:observations}.} we generate reasoning traces under each strategy and compute lexical, trajectory, and probabilistic anchoring. Unless otherwise noted, anchoring analyses use 5{,}000 evaluated examples per method. For trajectory anchoring, scalar raw tables report $100\times$ the sampled-prefix entropy gap defined in Appendix~\ref{app:trajectory_anchoring}, matching the percentage-point convention used for lexical and probabilistic anchoring.

\subsection{Behavioral Zones Construction}
\label{sec:behavioral_zones}

To empirically ground the anchoring measurement framework, we construct controlled reference conditions that isolate distinct regimes within the Trajectory--Probabilistic Anchoring plane (\cref{fig:behavioral_zones}). The horizontal axis is the reference-scaled trajectory anchoring score and the vertical axis is the reference-scaled probabilistic anchoring score. Thus the bottom-left corner represents low answer dependence in both the token-generation trajectory and the endpoint information contained in the trace, while the top-right corner represents severe answer leakage on both dimensions.

The four controls have two different origins. Two are real model texts: \textbf{Blind CoT} is a target-model reasoning trace generated from $Q$ alone, and \textbf{Response-as-CoT} uses the pre-committed response text $A$ itself as the trace. These form empirical low- and high-anchoring endpoints. The other two controls are mechanically constructed from the same $(Q,A)$ pairs, not proposed generation methods: \textbf{+Prob Anchor} is designed to increase answer recoverability from the trace while keeping the trace less token-deterministic from $A$, whereas \textbf{+Traj Anchor} is designed to make local token prediction strongly determined by $A$ while masking most answer content.

\begin{itemize}[leftmargin=*, itemsep=2pt]
    \item \textbf{Reason (Authentic Reasoning; bottom-left):} Blind CoT is generated without exposing the pre-committed answer to the model. It approximates standard forward reasoning and defines the low-anchoring corner: the trace is neither generated under answer-dependent token pressure nor especially informative about the hidden answer beyond what follows from the question.
    
    \item \textbf{Encode (Endpoint Information; top-left):} +Prob Anchor starts from a normal blind reasoning trace and appends an unordered tail of answer-derived evidence terms after neutral padding. This construction makes $A$ easier to predict from $R$ and therefore raises $\mathcal{A}_{\textup{prob}}$, but the unordered, padded evidence tail is not a step-by-step deterministic rendering of the answer. It therefore serves as a control for high endpoint information with weaker process-level dependence than direct copying.
    
    \item \textbf{Monitor (Process Anchoring; bottom-right):} +Traj Anchor uses a rule-based redaction of the answer prefix: function words and punctuation are preserved, a sparse subset of content words is copied, and the remaining content words are replaced by placeholders. Because the rule is deterministic given $A$, answer visibility makes the next tokens easier to predict and raises $\mathcal{A}_{\textup{traj}}$; because most semantic content is masked, the resulting trace carries much less recoverable endpoint information than the answer itself. This isolates the ``monitoring'' regime where the process is answer-steered even when explicit answer content is suppressed.
    
    \item \textbf{Copy (Severe Rationalization; top-right):} Response-as-CoT places the response text $A$ directly in the reasoning channel. It is the explicit leakage upper bound: the trace is both locally determined by the answer and maximally informative about the answer, so it anchors the high-high corner.
\end{itemize}

These controls separate the four quadrants used throughout the mechanism plots: Reason is low trajectory/low probabilistic anchoring, Encode is low trajectory/high probabilistic anchoring, Monitor is high trajectory/low probabilistic anchoring, and Copy is high on both axes. Individual method traces in \cref{fig:mechanism_diagnosis} are interpreted relative to this controlled calibration rather than by raw metric scale alone.

\subsection{Observations}
\label{sec:observations}

\paragraph{Suppression masks lexical anchoring but amplifies internal anchoring.}
As shown in \cref{tab:anchoring_metrics}, a dissociation exists between surface and latent anchoring effects. Suppression strategies (SUP, AUG-SUP) modestly reduce lexical anchoring ($\mathcal{A}_{\textup{lex}}$) compared to the neutral baseline, but diagnostic trajectory measurements show that they increase trajectory anchoring ($\mathcal{A}_{\textup{traj}}$)--the per-token generation dependence on the answer--while leaving probabilistic anchoring ($\mathcal{A}_{\textup{prob}}$) close to the neutral baseline. Stronger negative constraints may encourage the model to keep a more salient internal representation of the response precisely to avoid generating it explicitly; the model rationalizes more subtly, but more heavily.

\begin{table}[t]
\caption{Reference-normalized anchoring display scores by generation method. The columns use the same display scale as \cref{fig:mechanism_diagnosis}: color intensity for $\mathcal{A}_{\textup{lex}}$ and centroid coordinates for $\mathcal{A}_{\textup{traj}}$ and $\mathcal{A}_{\textup{prob}}$. Raw metric definitions are given above, and the display normalization is detailed in Appendix~\ref{app:display_score_normalization}.}
\label{tab:anchoring_metrics}
\centering
\small
\begin{tabular}{l ccc}
\toprule
\textbf{Method} & $\mathcal{A}_{\textup{lex}}^{\mathrm{disp}}$ $\downarrow$ & $\mathcal{A}_{\textup{traj}}^{\mathrm{disp}}$ $\downarrow$ & $\mathcal{A}_{\textup{prob}}^{\mathrm{disp}}$ $\downarrow$ \\
\midrule
NEU        & 41.9 & 47.5 & 52.3 \\
SUP        & 40.2 & 52.5 & 50.6 \\
AUG-SUP    & 40.5 & 52.7 & 49.4 \\
SSR        & \textbf{30.3} & \textbf{33.3} & \textbf{40.8} \\
\midrule
\textit{$\Delta$ (AUG-SUP vs.\ NEU)} & \textcolor{blue}{$-$3.3\%} & \textcolor{red}{+10.9\%} & \textcolor{blue}{$-$5.5\%} \\
\textit{$\Delta$ (SSR vs.\ NEU)}     & \textcolor{blue}{$-$27.6\%} & \textcolor{blue}{$-$29.8\%} & \textcolor{blue}{$-$21.9\%} \\
\bottomrule
\end{tabular}
\end{table}

\paragraph{Suppression induces pathological reasoning.}
Visualizing reasoning traces in the Trajectory and Probabilistic Anchoring plane reveals how suppression distorts the reasoning process (\cref{fig:mechanism_diagnosis}). Under NEU, traces already show substantial mass outside the Reason zone, indicating that default RCG is vulnerable to anchoring. Suppression (SUP, AUG-SUP) does not move traces into authentic reasoning: it slightly lowers Reason-zone mass and increases Monitor/Copy behavior. The model trades genuine derivation for constrained exploration or disguised encoding even when surface copying is reduced.

\paragraph{SSR achieves consistent reduction across all metrics.}
Unlike suppression methods, which trade lower lexical overlap for higher process-level answer dependence, SSR consistently reduces anchoring across all three levels (\cref{tab:anchoring_metrics}). By grounding generation in a pre-planned structural skeleton, SSR achieves the lowest lexical, trajectory, and probabilistic anchoring scores among the main methods. The reduction in lexical anchoring demonstrates effective surface-level mitigation, while the reduction in trajectory anchoring indicates that the generation process is less steered by the pre-committed response at each token, and the reduction in probabilistic anchoring confirms less total answer information in the trace.

\paragraph{Quality guardrail.}
Lower anchoring is useful only if the generated trace still supports the same abstract endpoint rather than drifting to a different task, stance, or artifact. We therefore report an abstract endpoint alignment guardrail in Appendix~\ref{app:endpoint_consistency}. This guardrail is deliberately abstract: it asks whether the trace targets the same user request, stance, and deliverable type as the reference answer, not whether it reproduces answer details. SSR is essentially indistinguishable from the neutral and suppression baselines on this guardrail ($0.989$ vs.\ $0.994$--$0.996$), indicating that its lower anchoring is not explained by endpoint drift.

\paragraph{SSR shifts traces toward more self-contained reasoning dynamics.}
Analyzing the behavioral zone distributions (\cref{fig:mechanism_diagnosis}), SSR counteracts the pathological shifts induced by suppression. While suppression leaves most traces in Encode, Monitor, or Copy zones, SSR achieves over half the traces in authentic reasoning (\cref{tab:behavioral_zones}). The structural skeleton helps avoid aimless generation by providing a plan and can discourage rationalization by defining granular step-wise intents. In the SSR panel, traces shift toward the bottom-left quadrant with reduced spread into pathological zones, indicating that SSR's structural guidance improves reasoning trajectories rather than merely shifting the mean while preserving high variance.

\begin{table}[t]
\caption{Distribution of reasoning traces across behavioral zones. Zones are read from the same reference-scaled display plane as \cref{fig:mechanism_diagnosis}, with a 0.5 threshold on trajectory and probabilistic anchoring. SSR maintains the highest proportion of authentic reasoning while minimizing pathological patterns.}
\label{tab:behavioral_zones}
\centering
\small
\begin{tabular}{l cccc}
\toprule
\textbf{Method} & \textbf{Reason} $\uparrow$ & \textbf{Encode} $\downarrow$ & \textbf{Monitor} $\downarrow$ & \textbf{Copy} $\downarrow$ \\
\midrule
NEU     & 38.0\% & 19.8\% & 26.9\% & 15.3\% \\
SUP     & 36.7\% & 16.8\% & 29.7\% & 16.8\% \\
AUG-SUP & 36.1\% & 17.3\% & 30.3\% & 16.3\% \\
SSR     & \textbf{54.8\%} & \textbf{16.3\%} & \textbf{18.2\%} & \textbf{10.7\%} \\
\bottomrule
\end{tabular}
\end{table}

\subsection{Interpretation}
\label{sec:interp}

The divergent behaviors of suppression-based methods and SSR can be understood through the lens of \emph{Ironic Process Theory} from cognitive psychology \citep{wegner1994ironic}. When instructed to suppress a concept (e.g., ``do not think of a white bear''), individuals paradoxically experience \emph{increased} accessibility of that concept. This occurs because suppression requires an active monitoring process to detect and exclude the forbidden content, a process that necessarily keeps the target concept cognitively salient.

We argue that an analogous phenomenon manifests in LLMs under semantic suppression. It is illustrated in \cref{fig:intro} (b) that when prompted to ``not reveal the answer,'' the model is likely to maintain a representation of what constitutes the answer to evaluate whether each generated token violates the constraint. This monitoring induces more anchoring effect, biasing the entire generation trajectory, even when the model successfully avoids surface-level copying.

SSR circumvents this paradox through a fundamentally different mechanism. Rather than imposing negative constraints that require answer monitoring, SSR provides a positive structural target that redirects the model's generative focus. As shown in \cref{fig:intro} (c), the skeleton specifies what operations to perform (e.g., \texttt{[PLAN]} $\rightarrow$ \texttt{[BRCH]} $\rightarrow$ \texttt{[EVAL]}) without encoding what results to obtain.

The skeleton provides a content-neutral structural target that reduces the need for continuous answer monitoring. Because generation is guided by the skeleton rather than by checking each token against the forbidden answer, trajectory anchoring decreases: the per-token entropy gap between answer-visible and answer-blind generation narrows. Meanwhile, the response-abstracted scaffold reduces probabilistic anchoring by channeling generation through a structural plan rather than through direct answer encoding, weakening the ironic cycle that suppression creates.

\begin{table*}[t]
    \caption{Main results across core benchmarks. Best results are in \textbf{bold}, and second-best are \underline{underlined}. SSR and SSR-D consistently outperform suppression-based methods, with SSR-D achieving the largest gains. Suppression methods show marginal improvement on ArenaHard but degrade multi-turn performance. (Arena: ArenaHard, EQ: EQ-Bench 3, IF: IFEval, MC: MultiChallenge)}
    \label{tab:main_results}
    \centering
    \small
    \setlength{\tabcolsep}{2.2pt}
    \begin{tabular*}{\textwidth}{@{\hspace{6pt}\extracolsep{\fill}} l cccc cccc cccc cccc @{\hspace{6pt}}}
        \toprule
        \multirow{2}{*}{\textbf{Method}} & \multicolumn{4}{c}{\textbf{Qwen3-8B-Think}} & \multicolumn{4}{c}{\textbf{Qwen3-32B-Think}} & \multicolumn{4}{c}{\textbf{Qwen3-4B-Think-2507}} & \multicolumn{4}{c}{\textbf{NBG4-3B-Base}} \\
        \addlinespace[3pt]
        & Arena & EQ & IF & MC & Arena & EQ & IF & MC & Arena & EQ & IF & MC & Arena & EQ & IF & MC \\
        \midrule
        NEU       & 50.8 & 76.4 & 80.0 & 38.2 & 68.2 & 85.5 & 83.4 & 43.8 & 41.7 & 79.3 & 76.7 & 38.9 & 32.4 & 79.6 & 74.3 & 34.5 \\
        SUP       & 51.5 & 77.5 & 81.2 & 35.5 & 68.8 & 85.6 & 84.5 & 41.5 & 42.3 & 79.2 & 78.5 & 36.6 & 32.9 & 81.5 & 76.5 & 33.4 \\
        AUG-SUP   & 52.8 & 78.3 & 81.5 & 33.0 & 69.1 & 86.4 & 83.2 & 41.3 & 43.5 & 79.0 & 80.1 & 34.5 & 33.2 & 82.8 & 77.2 & 33.0 \\
        \midrule
        SSR       & \underline{56.3} & \underline{82.6} & \underline{82.9} & \underline{39.4} & \underline{70.1} & \underline{87.7} & \underline{84.8} & \underline{44.6} & \underline{46.1} & \underline{81.9} & \underline{81.6} & \underline{40.2} & \underline{34.6} & \underline{84.4} & \underline{77.6} & \underline{35.8} \\
        SSR-D     & \textbf{59.5} & \textbf{86.1} & \textbf{83.7} & \textbf{41.2} & \textbf{72.0} & \textbf{89.6} & \textbf{85.1} & \textbf{45.0} & \textbf{50.8} & \textbf{83.1} & \textbf{83.8} & \textbf{42.0} & \textbf{37.2} & \textbf{86.5} & \textbf{78.2} & \textbf{37.1} \\
        \bottomrule
    \end{tabular*}
    \vskip -0.1in
\end{table*}

\begin{table}[t]
    \caption{Out-of-distribution performance on GPQA-Diamond and AIME 2025. \textbf{Bold} and \underline{underlined} values denote the best and second-best results among trained methods (excluding w/o Train).}
    \label{tab:ood_results}
    \centering
    \small
    \begin{tabular}{l cc cc}
        \toprule
        \multirow{2}{*}{\textbf{Method}} & \multicolumn{2}{c}{\textbf{Qwen3-8B-Think}} & \multicolumn{2}{c}{\textbf{Qwen3-32B-Think}} \\
        \cmidrule(lr){2-3} \cmidrule(lr){4-5}
        & GPQA-D & AIME & GPQA-D & AIME \\
        \midrule
        w/o Train & 58.1 & 70.0 & 64.1 & 73.3 \\
        \hdashline
        \addlinespace[3pt]
        NEU       & 49.5 & 33.3 & 60.6 & \underline{56.7} \\
        SUP       & 51.0 & 35.0 & 60.1 & 52.5 \\
        AUG-SUP   & 51.2 & 33.3 & 61.1 & 49.2 \\
        \midrule
        SSR & \underline{53.2} & \underline{40.0} & \underline{62.7} & \underline{56.7} \\
        SSR-D      & \textbf{56.6} & \textbf{44.2} & \textbf{65.2} & \textbf{59.2} \\
        \bottomrule
    \end{tabular}
    \vskip -0.1in
\end{table}

\section{Downstream Performance}

We next test the effectiveness of reverse CoT under different anchoring influences, addressing three questions: (1) Do reduced anchoring effects translate to improved downstream performance? (2) Does rationalization mitigation enable better out-of-distribution generalization? (3) Does distillation strengthen structural alignment beyond prompting?

The train-inference mismatch framework makes a specific prediction: traces with lower anchoring should produce better distillation outcomes, with the effect most pronounced on out-of-distribution tasks where pattern-matching to training examples cannot compensate for the mismatch. We test this prediction across in-distribution and OOD benchmarks.

\subsection{Experimental Setup}

\paragraph{Methods.}
\begin{itemize}
    \item We continue to evaluate the strategies defined in \cref{sec:method}: \textbf{NEU} (neutral prompting), \textbf{SUP} and \textbf{AUG-SUP} (semantic suppression baselines), \textbf{SSR} (structural skeleton-guided generation).
    \item \textbf{SSR-D} (distilled SSR). While SSR can be implemented via prompting, adherence to the structural format is often inconsistent, since models may omit tags or leak results into skeletons. SSR-D addresses this by fine-tuning the target model on SSR traces generated by a fine-tuned SSR teacher. Given a teacher-generated pair $(S^*, R^*)$, the target model is trained with two objectives: skeleton generation $\mathcal{L}_S = -\log p_\theta(S^* \mid Q, A)$ and reasoning reconstruction $\mathcal{L}_R = -\log p_\theta(R^* \mid Q, A, S^*)$. The combined objective $\mathcal{L} = \mathcal{L}_S + \mathcal{L}_R$ encourages both skeleton validity and skeleton-conditioned reasoning. Training details are provided in Appendix~\ref{app:training}.
\end{itemize}

\paragraph{Data.}
We sample 100k queries from LMArena \citep{chiang2024chatbot} and generate reference answers using Qwen3-Max. Given $(Q, A)$ pairs lacking intermediate reasoning, we generate reverse CoT using Qwen3-235B-Instruct-2507 \citep{qwen3} under each baseline condition, enabling evaluation at frontier-model scale. For SSR-D, a fine-tuned Qwen3-235B-Instruct-2507 SSR teacher serves as the generator of the distilled SSR traces.

\paragraph{Models.}
We fine-tune Qwen3-8B and Qwen3-32B \citep{qwen3} as representative open-weight thinking models, Qwen3-4B-Thinking-2507 for reasoning-specialized evaluation, and NBG4-3B-Base \citep{yang2025nanbeige4} to test cross-family transfer.

\paragraph{Benchmarks.}
We evaluate across diverse reasoning domains: open-ended reasoning \textbf{ArenaHard-v2.0} \citep{li2024crowdsourced}, human-aligned emotional intelligence reasoning \textbf{EQ-Bench 3} \citep{paech2023eqbench}, strict constraint adherence \textbf{IFEval} \citep{zhou2023instruction}, and multi-turn instruction-following \textbf{MultiChallenge} \citep{deshpande2025multichallenge}. For out-of-distribution evaluation, we use \textbf{GPQA-Diamond} \citep{rein2023gpqa} for scientific reasoning and \textbf{AIME 2025} \citep{aime25} for difficult mathematical reasoning. All evaluation details can be found in Appendix~\ref{app:benchmarks}.

\subsection{Main Results}

Table~\ref{tab:main_results} presents in-distribution results across four benchmarks and model scales.

\paragraph{Reduced anchoring yields consistent performance gains.}
SSR-D achieves the highest scores across all reported benchmarks and model scales, with consistent gains across model scales/families. On ArenaHard, improvements over NEU reach +9.1 points for Qwen3-4B-Thinking-2507, while MultiChallenge shows consistent gains up to +3.1 points. This correlation between lower internal anchoring (Table~\ref{tab:anchoring_metrics}) and higher downstream accuracy supports the usefulness of our framework: structured reasoning traces provide stronger training signals than post-hoc rationalizations.

\paragraph{Suppression degrades multi-turn reasoning.}
While SUP and AUG-SUP yield marginal single-turn improvements on ArenaHard, they consistently degrade MultiChallenge performance. This asymmetry is most pronounced for Qwen3-8B, where AUG-SUP underperforms NEU by 5 points despite appearing less anchored by $\mathcal{A}_{\text{lex}}$ (Figure~\ref{fig:intuition}). The pattern confirms that suppression-induced traces lack global coherence for sustained multi-turn reasoning, consistent with elevated $\mathcal{A}_{\text{traj}}$ in Table~\ref{tab:anchoring_metrics}.

\paragraph{Ruling out information loss as the driver.}
A potential concern is that SSR's lower $\mathcal{A}_{\textup{prob}}$ reflects reduced reasoning informativeness rather than reduced train-inference mismatch. Three lines of evidence argue against this. First, if information loss were the primary driver, downstream performance should decrease--yet SSR and SSR-D systematically outperform NEU in the main trained-method comparisons (Tables~\ref{tab:main_results}--\ref{tab:ood_results}). Second, behavioral zone analysis (Table~\ref{tab:behavioral_zones}) shows that SSR shifts traces into the Reason zone (38.0\% $\to$ 54.8\%), which corresponds to forward-reasoning behavior patterns; vacuous traces would not exhibit such patterns. Third, the abstract endpoint guardrail remains near-saturated for SSR (0.989 vs.\ 0.994--0.996 for NEU/SUP/AUG-SUP; Appendix~\ref{app:endpoint_consistency}), indicating that lower answer recoverability is not explained by drifting away from the intended task endpoint.

\paragraph{Distillation amplifies structural alignment.}
Comparing SSR and SSR-D reveals that prompting alone captures roughly half the potential benefit (e.g., +5 vs.\ +9 on ArenaHard for Qwen3-8B). The consistent SSR-D advantage reflects learned internalization of the skeleton pattern, ensuring reliable format adherence and stronger rationalization mitigation. Cross-architecture transfer to NBG4-3B-Base further suggests that the benefit is not limited to the Qwen model family.\footnote{We ablate the components of SSR in Appendix~\ref{app:ablation}.}

\subsection{Out-of-Distribution Generalization}

Table~\ref{tab:ood_results} reveals that standard RCG training can substantially degrade OOD performance, especially on AIME 2025. NEU-trained Qwen3-8B falls to less than half its untrained baseline, as the training data lack mathematical content comparable to competition-level problems. Suppression methods provide no recovery, while AUG-SUP even worsens AIME for Qwen3-32B by 7.5 points.

While all methods suffer from this domain gap, SSR and SSR-D substantially outperform suppression baselines through anchoring mitigation. For Qwen3-8B, SSR-D recovers about 30\% of the AIME gap to the untrained baseline, while Qwen3-32B with SSR-D exceeds its untrained GPQA-Diamond score. We attribute this advantage to structural decoupling: the model acquires transferable derivation patterns rather than content-specific shortcuts. Training data coverage determines the ceiling of OOD performance; anchoring mitigation determines how much of that ceiling is preserved.

\paragraph{Task-type dependence of mismatch costs.}
The severity of train-inference mismatch depends on task structure. On closed-form reasoning tasks with unique correct answers (GPQA-D, AIME), the mismatch is catastrophic: traces generated with answer access teach a shortcut reasoning style--linear, confident, non-exploratory--that collapses on unfamiliar problems. The most striking evidence is that NEU-trained Qwen3-8B drops below half its untrained AIME baseline (70.0 $\to$ 33.3), suggesting that high-PHR traces actively teach harmful reasoning patterns rather than merely failing to help. SSR-D partially recovers this loss (44.2), consistent with reduced mismatch enabling more transferable reasoning strategies.

On open-ended tasks (ArenaHard, EQ-Bench), the mismatch cost is more gradual: judge-based evaluation does not require exact answer recovery, so students can partially compensate by generating alternative but acceptable responses. Nevertheless, SSR-D's consistent gains (+9.1 ArenaHard for Qwen3-4B-Thinking-2507, +3.1 MultiChallenge over NEU) indicate that lower mismatch improves response quality even when evaluation is lenient.

\subsection{Additional Validation of RCoT Quality}

\paragraph{Filtering baselines and teacher-model bias.}
To compare against a strong forward-reasoning baseline, we evaluate a 16k subset on ArenaHard, where Qwen3-Max-Preview generates $r$ standard reasoning rollouts ($r=3$) and selects the best among them by self-judging using the same teacher as judge. When high-quality answers are available, ordinary RCG is not automatically stronger than repeated rollouts from a strong reasoning model: Best-of-3 filtering improves the base generator from 45.2 to 46.7, while NEU reaches 44.0. In contrast, SSR reaches 48.2 and SSR-D reaches 55.1 (\cref{tab:filtering_teacher_validation}), showing that structurally guided RCG can better leverage the given answer while preserving reasoning ability. To test whether the conclusion depends on Qwen3-Max as the sole teacher/evaluator, we repeat the same subset setting with Kimi2.5; SSR and SSR-D again outperform NEU, SUP, and AUG-SUP. This demonstrates that the gains are not specific to one teacher family.

\begin{table}[t]
\caption{Filtering and teacher-model validation on 16k subsets. The Qwen3-Max-Preview setting compares against standard forward-reasoning rollout filtering; the Kimi2.5 setting repeats data construction and evaluation with a different teacher/evaluator. Higher is better.}
\label{tab:filtering_teacher_validation}
\centering
\small
\begin{tabular}{lcc}
\toprule
\textbf{Method} & \textbf{Qwen3-Max-Preview} & \textbf{Kimi2.5} \\
\midrule
Native rollout & 45.2 & 37.8 \\
Best-of-3 rollout & 46.7 & 39.2 \\
\midrule
NEU       & 44.0 & 35.4 \\
SUP       & 45.5 & 36.9 \\
AUG-SUP   & 45.8 & 37.5 \\
\midrule
SSR       & \underline{48.2} & \underline{39.6} \\
SSR-D     & \textbf{55.1} & \textbf{44.0} \\
\bottomrule
\end{tabular}
\end{table}

\paragraph{External metric validation and qualitative evidence.}
A 500-sample LLM-as-judge study provides convergent external validation. In the \emph{self-contained derivation} pass, the judge sees only the query and reasoning trace and scores whether the trace can stand on its own. In the \emph{post-hoc dependence} pass, the judge additionally sees the pre-committed answer and scores how answer-driven the trace appears. Two independent judges, Claude Opus 4.6 (claude-opus-4-6, \citealp{anthropic2026claudeopus46}) and GPT-5.4 (gpt-5.4-2026-03-05, \citealp{openai2026gpt54}), consistently rank SSR-D best on self-contained derivation and lowest on post-hoc dependence, with SSR also ahead of SUP and NEU (\cref{tab:llm_judge_claude,tab:llm_judge_gpt}). Appendix~\ref{app:case_study} reports a qualitative case study, and Appendix~\ref{app:skeleton_structure} reports structural skeleton analysis.
In both tables, Var. denotes sample variance over the 500 judged examples.

\begin{table}[t]
\caption{LLM-as-judge validation with Claude Opus 4.6. Higher self-contained derivation and lower post-hoc dependence are better.}
\label{tab:llm_judge_claude}
\centering
\small
\begin{tabular}{lcccc}
\toprule
\textbf{Method} & \textbf{Deriv.} $\uparrow$ & \textbf{Var.} & \textbf{Dep.} $\downarrow$ & \textbf{Var.} \\
\midrule
NEU   & 4.375 & 0.890 & 3.563 & 1.217 \\
SUP   & 4.439 & 0.929 & 3.592 & 1.028 \\
SSR   & \underline{4.615} & 0.492 & \underline{3.448} & 0.776 \\
SSR-D & \textbf{4.744} & 0.263 & \textbf{2.698} & 0.684 \\
\bottomrule
\end{tabular}
\end{table}

\begin{table}[t]
\caption{LLM-as-judge validation with GPT-5.4. Higher self-contained derivation and lower post-hoc dependence are better.}
\label{tab:llm_judge_gpt}
\centering
\small
\begin{tabular}{lcccc}
\toprule
\textbf{Method} & \textbf{Deriv.} $\uparrow$ & \textbf{Var.} & \textbf{Dep.} $\downarrow$ & \textbf{Var.} \\
\midrule
NEU   & 3.580 & 1.458 & 4.480 & 0.919 \\
SUP   & 3.690 & 1.529 & 4.380 & 1.208 \\
SSR   & \underline{4.113} & 0.768 & \underline{4.155} & 1.445 \\
SSR-D & \textbf{4.410} & 0.366 & \textbf{3.040} & 1.312 \\
\bottomrule
\end{tabular}
\end{table}

\section{Related Work}

\paragraph{Reverse Chain-of-Thought Generation.}
While Chain-of-Thought (CoT) enables systematic reasoning \citep{wei2022chain, deepseekai2025deepseekr1,peng2025learn,feng2026pace}, supervised data often lacks intermediate traces. Existing solutions employ iterative self-improvement \citep{zelikman2022star}, reinforcement learning \citep{shao2024deepseekmath}, or trajectory reconstruction \citep{shridhar2023distilling, li2025from, wang2025reverse}. However, these methods expose the target answer during trace generation, inducing anchoring effects that compromise reasoning quality. We address this through structural decoupling, reducing direct dependence on answer-specific content.

\paragraph{Post-hoc Rationalizations in Chain-of-Thought.}
LLM rationales are not always faithful \citep{lanham2023measuring,paul2024making,tanneru2024hardness}. When models decide on a response before starting to reason, their post-hoc rationalizations may be influenced by unstated biases \citep{turpin2023language, lyu2023faithful} or contain deceptive shortcuts \citep{li2024deceptive, yee2024faithful,bentham2024chain}. Recent answer-attribution work further shows that LRM answers can arise from competing reasoning and retrieval mechanisms, and that retrieval-dominant behavior can produce post-hoc explanations for memorized answers \citep{wang2026reasoning}. While current mitigation methods focus on verification or process rewards \citep{lightman2024lets, wang2025stepwise,feng2025cosineverifier}, we frame the problem through an anchoring lens: answer visibility anchors generation as justification rather than derivation. We frame post-hoc rationalization as a source of train-inference mismatch: answer-visible generation produces traces whose distribution diverges from what students encounter at inference time. Our anchoring hierarchy quantifies this mismatch from surface tokens to generation-process dependence to total information transmission, and SSR mitigates it by construction through structural skeleton.

\paragraph{Structural Approaches in Reasoning.}
Prior work has established the importance of reasoning structure in LLM chain-of-thought \citep{madaan2023makes,li2025language,wang2025emergent,peng2025encode}. Structured reasoning has been obtained via non-linear exploration \citep{yao2023tree, besta2024got}, stage decomposition \citep{wang2023plan, zhou2023leasttomost,wen2025lock}, and meta-reasoning \citep{wang2024meta, zhang2025searching}. Recent long-CoT studies also show that fully constructed trajectories can underperform emergent teacher traces in standard forward-reasoning settings \citep{yang2025demystifying}. This is complementary to our finding rather than contradictory: SSR is not a hand-constructed full trajectory, but a lightweight response-abstracted scaffold for the answer-visible reverse-CoT setting. Our SSR leverages structure to redirect generation away from anchoring cues. Crucially, we constrain skeletons to avoid direct answer leakage, distinguishing our approach from efficiency-driven methods such as \citet{ning2024skeleton}.

\section{Conclusion}

In this work, we formalized post-hoc rationalization in the setting of Reverse Chain-of-Thought Generation (RCG), where visible pre-committed responses drive models toward rationalization rather than derivation. Our analysis revealed that standard semantic suppression paradoxically exacerbates this issue via an \textit{ironic process}, where the load of monitoring forbidden answers increases latent information leakage. We introduced Structural Skeleton-guided Reasoning (SSR), a paradigm that shifts from negative constraints to positive structural decoupling. By generating a response-abstracted functional skeleton, SSR provides a content-neutral structural target that can weaken the cycle of ironic monitoring. Experiments show that SSR reduces overall rationalization, while its applications in prompting and distillation consistently improve downstream task performance and improve out-of-distribution robustness in our evaluations. These results suggest that faithful reasoning may not be achieved by suppressing the undesirable, but by scaffolding the reasoning process to close the gap between how traces are generated--with answer access--and how students must reason--without it.

\section*{Acknowledgements}

This work was supported by the National Natural Science Foundation of China (62576010) and the Academic Research Projects of Beijing Union University (NO. ZK10202405).

\section*{Impact Statement}

This paper presents work whose goal is to advance the field of machine learning by improving the measurement and mitigation of post-hoc rationalization in generated reasoning traces. More faithful reasoning traces may benefit model evaluation, debugging, and reasoning distillation by reducing misleading explanations that merely justify pre-committed answers. The methods and metrics introduced here could also be misused to make model explanations appear more reliable than they are if reported without appropriate validation. We therefore encourage practitioners to pair anchoring metrics with task-specific correctness, robustness, and human evaluation when deploying reasoning-generation systems.

\bibliography{arxiv}
\bibliographystyle{icml2026}

\newpage
\appendix
\onecolumn

\section{Anchoring Metrics and Diagnostics}
\label{app:anchoring_details}

\subsection{Legacy Entropic Anchoring Metric}
\label{app:legacy_entropic_anchoring}

An earlier version of this work used an \emph{entropic anchoring} metric, denoted $\mathcal{A}_{\textup{ent}}$, as the middle level between lexical overlap and probabilistic answer recovery. We include the definition and diagnostic results here because it is a useful predecessor to the current trajectory metric, but it is no longer used as a primary anchoring measure. The original motivation came from the Uniform Information Density (UID) view of communication \citep{levy2006speakers,meister2021revisiting,tsipidi2024surprise} and from recent analyses of entropy profiles in LLM reasoning traces \citep{gwak2025revisiting}. The hypothesis was that post-hoc rationalized traces would have unnatural information-density dynamics: too little global exploration and abrupt local transitions where the trace repeatedly snaps back toward a pre-committed answer.

For a reasoning trace segmented into $N$ steps, let step $i$ contain $M_i$ tokens and let $p_t$ be the scorer model's next-token distribution at token position $t$. We first compute the step-level information density
\[
ID_i =
\frac{1}{M_i}
\sum_{t\in i}
\frac{-\sum_{v\in V} p_t(v)\log p_t(v)}{\log |V|}.
\]
Let $\widetilde{ID}=[\widetilde{ID}_1,\ldots,\widetilde{ID}_N]$ be the resulting step-density vector rescaled to the unit interval within a trace. The legacy metric combined a global-uniformity term with a local non-uniformity term:
\begin{align}
G_{\textup{unif}} &= \frac{1}{1+\mathrm{Var}(\widetilde{ID})/\tau_g},\\
\Delta_i &= |\widetilde{ID}_{i+1}-\widetilde{ID}_{i}|,\qquad
L_{\textup{non-unif}} =
\frac{\sigma_\Delta/\mu_\Delta}{1+\sigma_\Delta/\mu_\Delta},\\
\mathcal{A}_{\textup{ent}} &= \sqrt{G_{\textup{unif}}\cdot L_{\textup{non-unif}}}.
\end{align}
Intuitively, $G_{\textup{unif}}$ is high when the trace lacks broad entropy variation, and $L_{\textup{non-unif}}$ is high when adjacent steps have sharp entropy discontinuities. The old behavioral-zone figure used $\mathcal{A}_{\textup{ent}}$ as the horizontal axis and $\mathcal{A}_{\textup{prob}}$ as the vertical axis.

\begin{table}[t]
\caption{Legacy entropic anchoring scores from the earlier metric calibration and a blind-baseline rerun. Lower is better. The display column reproduces the earlier presentation, while the calibrated columns use raw excess scores after blind-baseline calibration.}
\label{tab:legacy_entropic_results}
\centering
\begin{tabular}{lccc}
\toprule
\textbf{Method} & $\mathcal{A}_{\textup{ent}}^{\mathrm{disp}}$ & $\mathcal{A}_{\textup{ent}}^{\mathrm{excess}}$ & $\mathcal{A}_{\textup{prob}}^{\mathrm{excess}}$ \\
\midrule
NEU     & 55.9 & 10.88 & 17.41 \\
SUP     & 57.4 & 10.65 & 16.69 \\
AUG-SUP & 58.4 & 11.53 & 16.19 \\
SSR     & --   &  8.14 &  8.49 \\
\bottomrule
\end{tabular}
\end{table}

The earlier display table appeared to support the intended story: SUP and AUG-SUP slightly increased $\mathcal{A}_{\textup{ent}}$ relative to NEU (55.9 $\to$ 57.4/58.4), while lowering lexical anchoring. However, the cleaner blind-baseline rerun weakened that interpretation. After blind-baseline calibration, SUP was essentially tied with NEU on $\mathcal{A}_{\textup{ent}}^{\mathrm{excess}}$ (10.65 vs.\ 10.88), AUG-SUP was only slightly higher (11.53), and neither suppression baseline increased total probabilistic excess. SSR remained lower on both entropic and probabilistic excess, but the suppression-specific ``ironic increase'' claim was no longer robust.

\paragraph{Why we retired $\mathcal{A}_{\textup{ent}}$.}
The main problem is that $\mathcal{A}_{\textup{ent}}$ is a trace-shape statistic, not an answer-conditional mismatch statistic. It measures the geometry of an entropy curve after the trace has already been generated; it does not compare answer-visible generation against answer-blind generation, nor does it ask whether the answer makes the next token easier to predict. A trace can receive a low entropic score because it is short, smooth, templatic, or gist-like, even if it still encodes the answer strongly; conversely, a richer exploratory trace can receive a high score without being more answer-anchored.

This concern appeared empirically in the legacy per-record diagnostics. Using the earlier per-trace min-max $\mathcal{A}_{\textup{ent}}$ records, reasoning length was strongly associated with the entropic score even after subtracting method means. On the four main methods (NEU, SUP, AUG-SUP, SSR), the method-residual correlation between reasoning tokens and $\mathcal{A}_{\textup{ent}}$ was Pearson $r=0.285$ and Spearman $\rho=0.720$; over a broader legacy baseline set, the corresponding residual correlations were $r=0.318$ and $\rho=0.574$. Method means showed the same risk: a long gist-like baseline had high $\mathcal{A}_{\textup{ent}}$ and high probabilistic encoding, while compact structured variants had very low $\mathcal{A}_{\textup{ent}}$ partly because the traces were shorter and smoother. These patterns made the metric too sensitive to length, segmentation, and formatting choices.

For this reason, the current paper keeps the useful intuition--answer access can distort the generation trajectory--but measures it directly with trajectory anchoring. The replacement metric $\mathcal{A}_{\textup{traj}}$ compares the next-token entropy under $(Q,R_{[:t]})$ against the entropy under $(Q,A,R_{[:t]})$, so it targets the train-inference mismatch itself: how much easier the already-generated trajectory becomes when the hidden answer is visible. In the final length diagnostic (Appendix~\ref{app:length_effect}), $\mathcal{A}_{\textup{traj}}$ has no positive length association on the main methods, unlike the retired entropic score.

\subsection{Trajectory Anchoring: Theoretical Foundations}
\label{app:trajectory_anchoring}

Trajectory anchoring measures the per-token influence of answer access on the generation process. For a fixed prefix $R_{[:t]}$, the entropy of the next-token distribution under two contexts gives:
\begin{align}
H_{\textup{blind}}(t) &= H\bigl(p(\cdot \mid Q, R_{[:t]})\bigr), \\
H_{\textup{aware}}(t) &= H\bigl(p(\cdot \mid Q, A, R_{[:t]})\bigr).
\end{align}
The gap $\Delta H(t) = H_{\textup{blind}}(t) - H_{\textup{aware}}(t)$ is non-negative in expectation by the information-theoretic property that conditioning reduces entropy. Thus
\[
\mathcal{A}_{\textup{traj}} = \frac{1}{|R|}\sum_t \Delta H(t)
\]
is zero when $A$ provides no additional next-token information beyond what $Q$ and $R_{[:t]}$ already provide, and grows as answer visibility increasingly steers the token trajectory.

\paragraph{Estimator used in experiments.}
Computing the entropy gap at every token is expensive for long traces, so the reported Qwen3-4B experiments use a fixed-prefix estimator. For each trace, we evaluate $\Delta H(t)$ at the token prefixes closest to 10\%, 25\%, 50\%, 75\%, and 90\% of the reasoning trace, and report
\[
\widehat{\mathcal{A}}_{\textup{traj}}
=
\frac{1}{|\Pi|}
\sum_{\pi\in\Pi}
\left[
H\bigl(p(\cdot \mid Q, R_{[:t_\pi]})\bigr)
-
H\bigl(p(\cdot \mid Q, A, R_{[:t_\pi]})\bigr)
\right],
\quad
\Pi=\{0.10,0.25,0.50,0.75,0.90\},
\quad
t_\pi=\lceil \pi |R| \rceil .
\]
This estimator is the trajectory-anchoring score used in scalar tables. It is a sampled-prefix estimate of the per-token definition above, not a separate metric.

\paragraph{Connection to train-inference mismatch.}
During RCG training, the model generates $R$ token-by-token under the answer-aware regime. During inference, the student generates under the answer-blind regime. The per-token entropy gap $\Delta H(t)$ directly quantifies how much more uncertain the student is at each step compared to the teacher that produced the training trace.

\paragraph{Relationship to other metrics.}
$\mathcal{A}_{\textup{traj}}$ is distinct from $\mathcal{A}_{\textup{prob}}$ in both direction and granularity. $\mathcal{A}_{\textup{prob}}$ measures $R \to A$ information: how much reading $R$ helps predict $A$. $\mathcal{A}_{\textup{traj}}$ measures $A \to R$ influence: how much knowing $A$ helps predict the next token of $R$. The two can dissociate. A trace may transmit substantial answer information while being generated with limited per-token answer steering if the information emerges from reasoning structure rather than direct monitoring.

\subsection{Abstract Endpoint Alignment Guardrail}
\label{app:endpoint_consistency}

Abstract endpoint alignment is a quality guardrail rather than an anchoring metric. It checks whether a generated trace still points to the same abstract endpoint as the reference assistant answer: the same user request, response act, stance, and deliverable or artifact type. This criterion deliberately does not require answer-specific wording, examples, entities, numbers, or code fragments to appear in the reasoning trace. A trace is marked inconsistent only when it drifts to a different request or turn, implies an incompatible stance or refusal behavior, targets a different deliverable, or becomes too generic to identify the intended endpoint.

We score this guardrail with the abstract endpoint judge prompt in the prompt appendix (Appendix~\ref{app:prompt}). For each valid judged example, the judge returns a binary consistency decision and a score in $\{1,3,5\}$. The reported alignment rate is the fraction of valid examples judged consistent. This diagnostic is reported separately from $\mathcal{A}_{\textup{lex}}$, $\mathcal{A}_{\textup{traj}}$, and $\mathcal{A}_{\textup{prob}}$, because it measures whether lower anchoring preserves the abstract answer-level target rather than how much answer information appears in the trace.

\begin{table}[h]
\caption{Abstract endpoint alignment guardrail. Higher is better. The main methods are all near-saturated, showing that SSR's anchoring reduction is not caused by endpoint drift.}
\label{tab:endpoint_consistency}
\centering
\begin{tabular}{lcc}
\toprule
\textbf{Method} & \textbf{Alignment rate} $\uparrow$ & \textbf{Mean score} $\uparrow$ \\
\midrule
NEU & 0.996 & 4.984 \\
SUP & 0.996 & 4.984 \\
AUG-SUP & 0.994 & 4.976 \\
SSR & 0.989 & 4.956 \\
\midrule
QA-SUP & 0.993 & 4.972 \\
PG-SUP & 0.994 & 4.976 \\
FDB & 0.955 & 4.818 \\
Gist & 0.999 & 4.996 \\
BoN & 0.994 & 4.976 \\
NGramBlock & 0.926 & 4.701 \\
\bottomrule
\end{tabular}
\end{table}

The main methods satisfy this guardrail at very high and similar rates. SSR remains in the same near-saturated regime as NEU/SUP/AUG-SUP, so its lower anchoring should not be read as a failure to preserve the abstract answer endpoint. Lower scores for FDB and NGramBlock illustrate why the guardrail is still useful: methods that alter the trace aggressively can become less endpoint-specific even when they are informative for anchoring diagnostics.

\subsection{RCoT Length and Anchoring Metrics}
\label{app:length_effect}

A possible confound is that longer RCoT traces have more opportunities to mention answer-related content or to reduce answer uncertainty, making lower anchoring appear to be a mere length artifact. We therefore measure the per-example relationship between reasoning length and the three anchoring metrics on the final Qwen3-4B fixed method set. To avoid conflating method identity with length, we report method-residual correlations over the four main methods (NEU, SUP, AUG-SUP, SSR): for each method, both reasoning length and the metric value are centered by that method's mean before computing the correlation. Here $\mathcal{A}_{\textup{lex}}$ uses the final question-filtered lexical score.

\begin{table}[h]
\caption{Correlation between RCoT length and anchoring metrics for the main four methods ($n=4{,}000$ traces). Values are computed after removing each method's mean length and mean metric value. Slopes are reported in the same percentage-point units used by scalar raw anchoring tables: $100\times\mathcal{A}_{\textup{lex}}$, $100\times\mathcal{A}_{\textup{traj}}$, and $100\times\mathcal{A}_{\textup{prob}}$ per 100 generated reasoning tokens.}
\label{tab:length_effect}
\centering
\begin{tabular}{lccc}
\toprule
\textbf{Metric} & \textbf{Pearson $r$} & \textbf{Spearman $\rho$} & \textbf{Slope / 100 tok.} \\
\midrule
$\mathcal{A}_{\textup{lex}}$  & 0.225  & 0.355  & 1.1 \\
$\mathcal{A}_{\textup{traj}}$ & -0.050 & -0.073 & -0.6 \\
$\mathcal{A}_{\textup{prob}}$ & 0.209  & 0.266  & 0.7 \\
\bottomrule
\end{tabular}
\end{table}

The diagnostic shows a weak positive length association for lexical and probabilistic anchoring, but no positive association for trajectory anchoring. This effect is too small to explain SSR's lower anchoring. For example, SSR traces are only 23.3 tokens shorter than NEU on average; the residual length slope predicts decreases of only 0.252 points in $\mathcal{A}_{\textup{lex}}$ and 0.152 points in $\mathcal{A}_{\textup{prob}}$, whereas the observed reductions are 3.696 and 1.222 points, respectively. For $\mathcal{A}_{\textup{traj}}$, the residual length slope predicts a small increase of about 0.14 points for the shorter SSR traces, whereas the observed SSR--NEU reduction is 8.714 points. Moreover, SSR is slightly longer than SUP and AUG-SUP, so length would predict slightly higher lexical/probabilistic anchoring for SSR relative to those suppression baselines, opposite to the observed pattern. We therefore treat RCoT length as a mild covariate for surface and endpoint predictability, not as the driver of the main anchoring improvements.

\subsection{Reference-Normalized Display Scores}
\label{app:display_score_normalization}

The underlying metrics above have different native scales: lexical anchoring is a question-filtered overlap ratio, trajectory anchoring is an entropy/confidence-gap diagnostic, and probabilistic anchoring is a clipped answer-surprisal reduction. In scalar raw tables, all three are reported as percentage-point scores: $100\times\mathcal{A}_{\textup{lex}}$, $100\times\mathcal{A}_{\textup{traj}}$, and $100\times\mathcal{A}_{\textup{prob}}$. For the behavioral-zone figures and the main scalar anchoring table, we additionally report reference-normalized display scores on a common $0$--$100$ scale.

For each raw metric $m\in\{\mathcal{A}_{\textup{lex}},\mathcal{A}_{\textup{traj}},\mathcal{A}_{\textup{prob}}\}$, let
\[
L_m=\overline{m}(\text{Blind CoT}), \qquad
M_m=\operatorname{median}_{g\in\{\mathrm{NEU},\mathrm{SUP},\mathrm{AUG\mbox{-}SUP},\mathrm{SSR}\}}\overline{m}(g),
\]
where $\overline{m}(g)$ is the method-level mean in the Qwen3-4B-Thinking-2507 anchoring evaluation. We set the symmetric saturation point
\[
U_m = L_m + 2(M_m-L_m)
\]
and transform a raw per-record score $x$ by
\[
D_m(x)=100\cdot\operatorname{clip}_{[0,1]}\left(\frac{x-L_m}{U_m-L_m}\right).
\]
Thus Blind CoT anchors the lower end of the display scale, the median main-method raw score maps to $50$, and values at or above $U_m$ saturate at $100$. This is a visualization and reporting scale; the raw metric definitions remain those in the main text.

For the main Qwen3-4B-Thinking-2507 calibration used in \cref{fig:behavioral_zones,fig:mechanism_diagnosis,tab:anchoring_metrics}, the reference triples $(L_m,M_m,U_m)$ on this percentage-point raw scale are $(5.7466,12.7747,19.8027)$ for $\mathcal{A}_{\textup{lex}}$, $(2.3013,33.0909,63.8805)$ for $\mathcal{A}_{\textup{traj}}$, and $(0.2572,5.5812,10.9051)$ for $\mathcal{A}_{\textup{prob}}$. Additional scalar anchoring tables reuse this calibration for comparability; behavioral-zone coordinate tables may separately report model-specific controlled coordinates when they are used for zone assignment.

The aggregation matches \cref{fig:mechanism_diagnosis}. For $\mathcal{A}_{\textup{traj}}$ and $\mathcal{A}_{\textup{prob}}$, the reported scalar is the method-centroid display coordinate, $D_m(\overline{m}(g))$. For $\mathcal{A}_{\textup{lex}}$, the figure uses per-record color intensity, so the table reports the mean displayed color score, $\mathbb{E}_{i\in g}[D_{\textup{lex}}(m_i)]$. This is why the lexical display score is not generally identical to applying $D_{\textup{lex}}$ after taking the raw method mean.

\subsection{Additional Anchoring Results on Qwen3-8B}
\label{app:observations}

Beyond Qwen3-4B-Thinking-2507, we also test the observations in \cref{sec:observations} on Qwen3-8B. This cross-model audit uses the final SSR prompt and reports the same question-filtered lexical anchoring, trajectory anchoring, probabilistic anchoring, behavioral-zone statistics, and endpoint guardrail. The Blind CoT endpoint is generated independently from the question-only setting and rescored with Qwen3-8B, matching the Qwen3-4B construction rather than reusing NEU traces. The pattern is consistent with the main Qwen3-4B analysis: suppression leaves process-level anchoring slightly above NEU, while SSR reduces all three anchoring metrics and shifts more traces into the Reason zone.

\begin{figure}[t]
  \centering
  \includegraphics[width=0.31\linewidth]{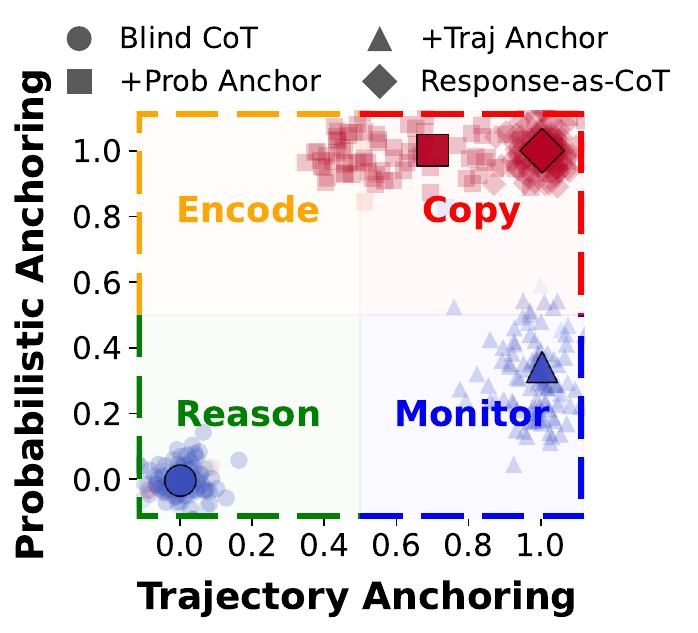}
  \caption{Controlled reference calibration for the Qwen3-8B behavioral-zone display. The corrected Blind CoT endpoint uses independent question-only traces, while Response-as-CoT, +Prob Anchor, and +Traj Anchor follow the same construction as the Qwen3-4B analysis.}
  \label{fig:behavioral_zones_qwen3_8b}
\end{figure}

\begin{figure*}[t]
  \centering
  \includegraphics[width=0.96\textwidth]{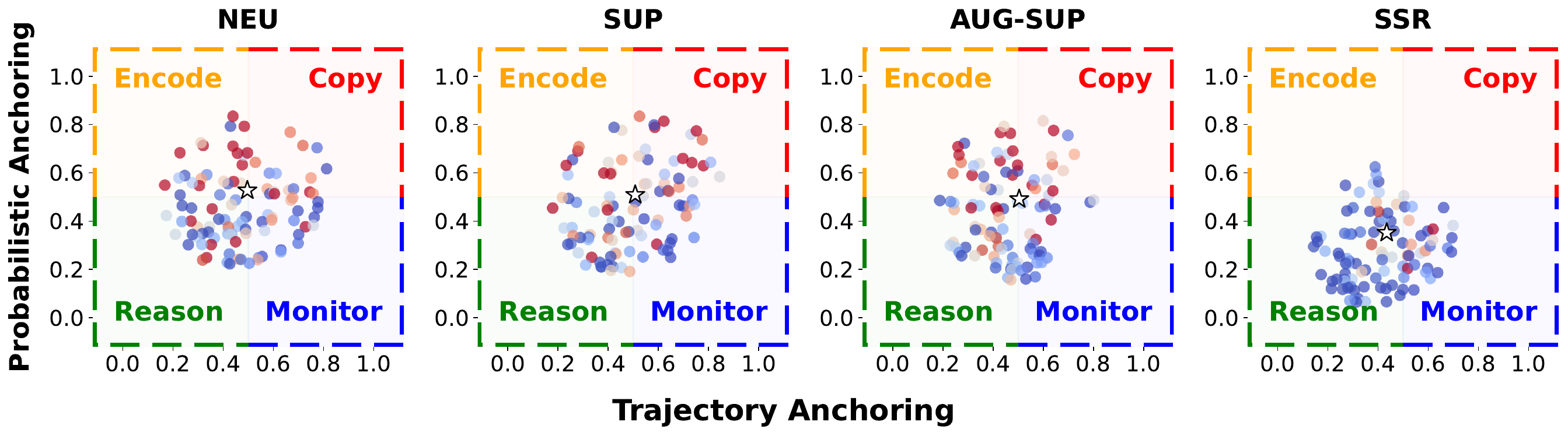}
  \caption{Mechanism diagnosis across generation strategies on Qwen3-8B, using the corrected Blind CoT endpoint. SSR shifts traces toward lower trajectory and probabilistic anchoring compared with NEU, SUP, and AUG-SUP.}
  \label{fig:mechanism_diagnosis_qwen3_8b}
\end{figure*}

\begin{table}[t]
\caption{Additional reference-normalized anchoring display scores by generation method on Qwen3-8B. The table uses the same display procedure as the main Qwen3-4B analysis, with the corrected Qwen3-8B Blind CoT endpoint as the low reference.}
\label{tab:anchoring_metrics_qwen3_8b}
\centering
\begin{tabular}{l ccc}
\toprule
\textbf{Method} & $\mathcal{A}_{\textup{lex}}^{\mathrm{disp}}$ $\downarrow$ & $\mathcal{A}_{\textup{traj}}^{\mathrm{disp}}$ $\downarrow$ & $\mathcal{A}_{\textup{prob}}^{\mathrm{disp}}$ $\downarrow$ \\
\midrule
NEU        & 41.9 & 49.6 & 52.7 \\
SUP        & 40.2 & 50.8 & 50.9 \\
AUG-SUP    & 40.5 & 50.4 & 49.1 \\
SSR        & \textbf{30.3} & \textbf{43.4} & \textbf{35.2} \\
\midrule
\textit{$\Delta$ (AUG-SUP vs.\ NEU)} & \textcolor{blue}{$-3.3\%$} & \textcolor{red}{+1.6\%} & \textcolor{blue}{$-6.9\%$} \\
\textit{$\Delta$ (SSR vs.\ NEU)}     & \textcolor{blue}{$-27.6\%$} & \textcolor{blue}{$-12.4\%$} & \textcolor{blue}{$-33.3\%$} \\
\bottomrule
\end{tabular}
\end{table}

\begin{table}[t]
\caption{Controlled reference conditions for Qwen3-8B behavioral-zone construction. The first three columns and the $x,y$ coordinates use the same display plane as \cref{fig:mechanism_diagnosis_qwen3_8b}, with $x=\mathcal{A}_{\textup{traj}}$ and $y=\mathcal{A}_{\textup{prob}}$.}
\label{tab:controlled_anchors_qwen3_8b}
\centering
\begin{tabular}{lccccc}
\toprule
\textbf{Condition} & $\mathcal{A}_{\textup{lex}}^{\mathrm{disp}}$ & $\mathcal{A}_{\textup{traj}}^{\mathrm{disp}}$ & $\mathcal{A}_{\textup{prob}}^{\mathrm{disp}}$ & $x$ & $y$ \\
\midrule
Blind CoT        & 15.4 & 0.0  & 0.0   & 0.00 & 0.00 \\
+Prob Anchor     & 93.8 & 69.8 & 100.0 & 0.70 & 1.00 \\
+Traj Anchor     & 6.8  & 100.0 & 34.5 & 1.00 & 0.34 \\
Response-as-CoT  & 100.0 & 100.0 & 100.0 & 1.00 & 1.00 \\
\bottomrule
\end{tabular}
\end{table}

\begin{table}[t]
\caption{Distribution of reasoning traces across behavioral zones on Qwen3-8B, using the corrected Blind CoT endpoint and the same display-plane thresholds as \cref{fig:mechanism_diagnosis_qwen3_8b}.}
\label{tab:behavioral_zones_qwen3_8b}
\centering
\begin{tabular}{l cccc}
\toprule
\textbf{Method} & \textbf{Reason} $\uparrow$ & \textbf{Encode} $\downarrow$ & \textbf{Monitor} $\downarrow$ & \textbf{Copy} $\downarrow$ \\
\midrule
NEU     & 37.6\% & 19.4\% & 26.3\% & 16.7\% \\
SUP     & 38.9\% & 16.1\% & 26.6\% & 18.4\% \\
AUG-SUP & 38.8\% & 18.7\% & 26.8\% & 15.7\% \\
SSR     & \textbf{50.6\%} & \textbf{12.7\%} & \textbf{25.3\%} & \textbf{11.4\%} \\
\bottomrule
\end{tabular}
\end{table}

\begin{table}[t]
\caption{Abstract endpoint alignment guardrail for the Qwen3-8B cross-model audit. Higher is better.}
\label{tab:endpoint_consistency_qwen3_8b}
\centering
\begin{tabular}{lcc}
\toprule
\textbf{Method} & \textbf{Alignment rate} $\uparrow$ & \textbf{Mean score} $\uparrow$ \\
\midrule
NEU & 0.996 & 4.984 \\
SUP & 0.996 & 4.984 \\
AUG-SUP & 0.994 & 4.976 \\
SSR & 0.989 & 4.956 \\
\bottomrule
\end{tabular}
\end{table}

\subsection{Extended Mismatch Metrics}
\label{app:extended_mismatch_metrics}

We further report diagnostic mismatch metrics from the Qwen3-4B-Thinking-2507 scorer. Here $\mathcal{A}_{\textup{traj}}$ corresponds to the sampled-prefix entropy gap, $B(1.00)$ is the raw endpoint bit gain corresponding to probabilistic anchoring before normalization, commitment concentration summarizes how concentrated answer information is over the trace, and KL/JSD provide alternative next-token distributional comparisons.

\begin{table*}[t]
\caption{Extended mismatch metrics on Qwen3-4B-Thinking-2507. The $\mathcal{A}_{\textup{traj}}$ column reports $100\times$ the sampled-prefix entropy gap for compact display; $B(1.00)$ is an absolute endpoint bit-gain diagnostic and keeps its original units. Lower $B(1.00)$, $\mathcal{A}_{\textup{traj}}$, KL, and JSD indicate lower answer dependence; commitment concentration is descriptive and not optimized directly.}
\label{tab:extended_mismatch_metrics}
\centering
\begin{tabular}{lccccc}
\toprule
\textbf{Method} & $B(1.00)$ $\downarrow$ & $\mathcal{A}_{\textup{traj}}$ (\%) $\downarrow$ & \textbf{Commitment Conc.} & \textbf{KL} $\downarrow$ & \textbf{JSD} $\downarrow$ \\
\midrule
NEU & 0.1412 & 31.54 & 0.5516 & 0.9608 & 0.1371 \\
SUP & 0.1353 & 34.65 & 0.5616 & 0.9980 & 0.1431 \\
AUG-SUP & 0.1331 & 34.73 & 0.5610 & 0.9900 & 0.1438 \\
SSR & 0.1114 & 22.83 & 0.5592 & 0.6479 & 0.1020 \\
QA-SUP & 0.1926 & 30.43 & 0.5320 & 0.8757 & 0.1283 \\
PG-SUP & 0.2164 & 29.91 & 0.5041 & 0.8376 & 0.1216 \\
FDB & 0.0253 & 1.83 & 0.6283 & 0.5552 & 0.0796 \\
Gist & 0.2607 & 27.48 & 0.4338 & 0.6673 & 0.1030 \\
BoN & 0.0905 & 31.08 & 0.5807 & 0.9020 & 0.1319 \\
NGramBlock & 0.0942 & 26.62 & 0.5741 & 0.8563 & 0.1229 \\
\bottomrule
\end{tabular}
\end{table*}

\subsection{Commitment Profile Analysis}
\label{app:commitment_profile}

We analyze the temporal distribution of answer information injection by computing $B(\tau)=\log P(A\mid Q,R_{[:\tau]})-\log P(A\mid Q)$ at prefix checkpoints. The endpoint $B(1.00)$ equals the raw endpoint bit gain. The Qwen3-4B diagnostic profile shows that suppression does not meaningfully lower total answer information relative to NEU, whereas SSR reduces the endpoint bit gain and the middle-to-late prefix gains.

\begin{table*}[t]
\caption{Raw prefix bit-gain profile $B(\tau)$ on Qwen3-4B-Thinking-2507. Values are absolute bit gains, not ratios.}
\label{tab:commitment_profile}
\centering
\begin{tabular}{lcccccc}
\toprule
\textbf{Method} & $B(0.10)$ & $B(0.25)$ & $B(0.50)$ & $B(0.75)$ & $B(0.90)$ & $B(1.00)$ \\
\midrule
NEU & 0.0178 & 0.0142 & 0.0552 & 0.1097 & 0.1326 & 0.1412 \\
SUP & 0.0197 & 0.0154 & 0.0540 & 0.1056 & 0.1268 & 0.1353 \\
AUG-SUP & 0.0179 & 0.0165 & 0.0543 & 0.1046 & 0.1236 & 0.1331 \\
SSR & 0.0349 & 0.0365 & 0.0646 & 0.0943 & 0.1003 & 0.1114 \\
QA-SUP & 0.0209 & 0.0237 & 0.0784 & 0.1460 & 0.1783 & 0.1926 \\
PG-SUP & 0.0223 & 0.0302 & 0.0890 & 0.1586 & 0.1952 & 0.2164 \\
FDB & -0.0004 & -0.0154 & -0.0196 & -0.0138 & 0.0081 & 0.0253 \\
Gist & 0.0442 & 0.0689 & 0.1378 & 0.2073 & 0.2408 & 0.2607 \\
BoN & 0.0191 & 0.0124 & 0.0359 & 0.0729 & 0.0845 & 0.0905 \\
NGramBlock & 0.0177 & 0.0133 & 0.0393 & 0.0759 & 0.0890 & 0.0942 \\
\bottomrule
\end{tabular}
\end{table*}

Ratio metrics can be misleading when total answer information differs across methods, so we report absolute bit gains. Under the Qwen3-4B scorer, SSR has higher very-early $B(0.10)$ and $B(0.25)$ than NEU, but by late prefixes and the endpoint it transmits substantially less answer information (e.g., $B(1.0)=0.1114$ vs.\ $0.1412$), indicating that the benefit is concentrated in the middle-to-late trace rather than in the first prefix checkpoints.

\subsection{Path Diversity Compression}
\label{app:path_diversity}

As a complementary diagnostic, we measure whether answer visibility changes the diversity of possible reasoning paths. For each query, we sample $k=5$ answer-conditioned traces and $k=5$ answer-blind traces, compute the mean pairwise embedding distance within each set, and report
\[
\mathrm{DivCompress}
=
\frac{\mathrm{Diversity}(R_1,\ldots,R_k \sim P(R\mid Q,A))}
{\mathrm{Diversity}(R^*_1,\ldots,R^*_k \sim P(R\mid Q))}.
\]
Values close to 1 indicate that answer-conditioned generation preserves roughly the same path diversity as answer-blind reasoning, while larger values indicate that answer-conditioned traces occupy a broader or more variable reasoning-path set under this embedding metric. After filtering incomplete sample sets, the resulting ratios are highly heavy-tailed because a small number of answer-blind sample sets have near-zero embedding diversity. For this reason, we treat path diversity as a stress diagnostic rather than a primary quantitative claim: the main methods show NEU $166.3$ [10.6, 459.5], SUP $169.9$ [16.0, 465.5], AUG-SUP $327.9$ [12.8, 812.7], and SSR $97.9$ [6.6, 219.2], but the wide intervals caution against fine-grained ranking.

\section{SSR Theory and Implementation}
\label{app:ssr_theory}

We provide an information-theoretic framework for the \emph{skeleton-mediated} channel of Structural Skeleton-guided Reasoning (SSR). The goal of this appendix is deliberately limited: we bound the information about the pre-committed response $A$ that can be transmitted through the skeleton $S$. We do not claim a formal bound on the residual reasoning-generation channel $I(A;R\mid Q,S)$, since the second phase still conditions on $A$. The overall reduction of probabilistic anchoring is therefore an empirical result, measured by $\mathcal{A}_{\mathrm{prob}}$ in the main experiments.

\subsection{Skeleton Capacity and Functional Invariance}

Let a structural skeleton
\[
S = \langle (F_i, C_i) \rangle_{i=1}^{N}
\]
consist of functional tags $F_i \in \mathcal{F}$ and content summaries $C_i$. All information quantities in this appendix are measured in bits. We first work in the fixed-length setting $N=n$, and later recover the variable-length case by marginalizing over $N$. Under $N=n$, let
\[
S_{<i} = \bigl((F_1,C_1),\ldots,(F_{i-1},C_{i-1})\bigr)
\]
denote the autoregressive skeleton history before step $i$, and let
\[
H_i^{(n)} = (Q, N=n, S_{<i}, F_i)
\]
be the random context used in the fixed-length analysis for the $i$-th content summary. Conditioning on $N=n$ is an analytical convention; it does not assert that the model observes the final length during generation.

\begin{definition}[Sequential $\epsilon$-Functional Invariance]
\label{def:func_inv}
A skeleton generator is \emph{sequentially $\epsilon$-functionally invariant} under fixed length $N=n$ if each content-summary step satisfies
\[
I(A; C_i \mid H_i^{(n)}) \leq \epsilon_i^{(n)} .
\]
Here $\epsilon_i^{(n)}\in\mathbb{R}_{\ge 0}$ is a scalar conditional-MI bound after marginalizing over the random context $H_i^{(n)}$, so these bounds compose additively under the chain rule below. We use this average mutual-information form unless stated otherwise.
\end{definition}

A stronger pointwise sufficient condition is that, for all $(h_i,a)$ in the support of $(H_i^{(n)},A)$,
\[
D_{\mathrm{KL}}^{(2)}\!\left(
P(C_i\mid H_i^{(n)}=h_i,A=a)
\,\|\, 
P(C_i\mid H_i^{(n)}=h_i)
\right)
\leq \epsilon_i^{(n)},
\]
where $D_{\mathrm{KL}}^{(2)}$ uses base-2 logarithms. This pointwise condition implies Definition~\ref{def:func_inv}, but is stronger than the average conditional-MI requirement: averaging the pointwise KL over $(H_i^{(n)},A)$ gives $I(A;C_i\mid H_i^{(n)})\leq \epsilon_i^{(n)}$. Although $\epsilon_i^{(n)}$ is a theoretical primitive, it can be empirically diagnosed through held-out estimates of the pointwise bit gain $\log_2[P(C_i\mid H_i^{(n)},A)/P(C_i\mid H_i^{(n)})]$, or by probing summaries for $A$-relevant features.

\begin{proposition}[Skeleton-Channel Bound]
\label{prop:skeleton_channel_bound}
For a fixed-length $n$-step skeleton with discrete functional tags $F_i \in \mathcal{F}$, if the content summaries satisfy sequential $\epsilon$-functional invariance, then
\[
I(A; S \mid Q,N=n)
\leq
\sum_{i=1}^{n} H(F_i\mid Q,N=n,S_{<i})
+ \sum_{i=1}^{n}\epsilon_i^{(n)}
\leq
n\log_2|\mathcal{F}|+\sum_{i=1}^{n}\epsilon_i^{(n)} .
\]
For variable-length skeletons,
\[
I(A; S \mid Q)
= I(A;N\mid Q)+\mathbb{E}_{n\sim N}\!\left[I(A;S\mid Q,N=n)\right]
\]
where each fixed-length term is evaluated under the conditional law given $N=n$,
and hence
\[
I(A; S \mid Q)
\leq
I(A;N\mid Q)
+
\mathbb{E}_{n\sim N}\!\left[
n\log_2|\mathcal{F}|+
\sum_{i=1}^{n}\epsilon_i^{(n)}
\right].
\]
In particular, if skeleton length is bounded by $N_{\max}$, then $I(A;N\mid Q)\leq H(N\mid Q)\leq \log_2 N_{\max}$, so the length channel contributes at most a small additive term relative to long free-form responses.
\end{proposition}

\begin{proof}
For fixed $N=n$, apply the chain rule in the natural autoregressive skeleton order:
\[
I(A;S\mid Q,N=n)
=
\sum_{i=1}^{n}
I(A;F_i,C_i\mid Q,N=n,S_{<i}).
\]
Expanding each pair $(F_i,C_i)$ gives
\[
I(A;S\mid Q,N=n)
=
\sum_{i=1}^{n}
I(A;F_i\mid Q,N=n,S_{<i})
+
\sum_{i=1}^{n}
I(A;C_i\mid Q,N=n,S_{<i},F_i).
\]
The tag terms are bounded by conditional entropy:
\[
I(A;F_i\mid Q,N=n,S_{<i})
\leq H(F_i\mid Q,N=n,S_{<i})
\leq \log_2|\mathcal{F}|.
\]
The content-summary terms are bounded directly by Definition~\ref{def:func_inv}, since $H_i^{(n)}=(Q,N=n,S_{<i},F_i)$:
\[
I(A;C_i\mid Q,N=n,S_{<i},F_i)
= I(A;C_i\mid H_i^{(n)})
\leq \epsilon_i^{(n)}.
\]
Summing over $i$ proves the fixed-length bound. For the variable-length case, $N$ is a deterministic function of the skeleton $S$, so the chain rule gives
\[
I(A;S\mid Q)
=I(A;N,S\mid Q)
=I(A;N\mid Q)+I(A;S\mid Q,N).
\]
Equivalently,
\[
I(A;S\mid Q,N)=\mathbb{E}_{n\sim N}\!\left[I(A;S\mid Q,N=n)\right].
\]
Applying the fixed-length bound for each realized length yields the stated result.
\end{proof}

\begin{remark}[No free lunch in invariance]
The useful operating regime is not $\epsilon_i^{(n)}\to 0$ at all costs. If the skeleton transmits no answer-relevant guidance through tag order, length, or abstract operations, it may become too generic to support the intended reasoning path, pushing answer dependence back into the residual term $I(A;R\mid Q,S)$ during reasoning generation. Formally, the decomposition below means that driving $I(A;S\mid Q)$ toward zero does not automatically reduce total anchoring: the residual $I(A;R\mid Q,S)$ can grow toward $H(A\mid Q,S)$ when $S$ stops carrying useful guidance. SSR therefore aims for a low-bandwidth structural channel: answer-relevant guidance is routed through coarse structure, while concrete result leakage in content summaries is discouraged.
\end{remark}

\subsection{Connection to the Anchoring Measurement Hierarchy}

\paragraph{Lexical anchoring.}
Content summaries satisfying functional invariance describe \emph{what operation to perform} rather than \emph{what result to obtain} (e.g., ``calculate the ratio'' rather than ``calculate $0.5$''). This design discourages direct lexical overlap between the skeleton and the pre-committed response. The final reasoning trace can still leak answer content during phase two, so lexical anchoring remains an empirical metric rather than a theorem-level consequence.

\paragraph{Trajectory anchoring.}
The structural skeleton specifies a reasoning topology through functional tags (e.g., \texttt{[PLAN]} $\to$ \texttt{[BRCH]} $\to$ \texttt{[EVAL]}). Exploration tags such as \texttt{[BRCH]} and \texttt{[RFLX]} can create natural branching and consolidation points, while \texttt{[EVAL]} and \texttt{[SUMM]} can mark verification phases. This provides a qualitative mechanism for reducing per-token answer monitoring: the next token can be guided by the structural target rather than by direct comparison against the hidden answer. This is not a formal guarantee on $\mathcal{A}_{\mathrm{traj}}$, since phase two still conditions on $A$, but it motivates why SSR should reduce the answer-aware versus answer-blind entropy gap.

\paragraph{Probabilistic anchoring.}
The skeleton-channel bound above most directly informs the endpoint information metric. Probabilistic anchoring $\mathcal{A}_{\mathrm{prob}}$ measures the clipped fraction of baseline answer surprisal removed by the reasoning trace:
\[
\mathcal{A}_{\mathrm{prob}}(q,r,a)
=
\operatorname{clip}_{[0,1]}\!\left(
\frac{\ell_\theta(a \mid q,r)-\ell_\theta(a \mid q)}
{-\ell_\theta(a \mid q)}
\right),
\quad
\ell_\theta(a \mid c)=\frac{1}{|a|}\log_2 P_\theta(a\mid c).
\]
This score should be read as a scorer-model proxy for endpoint recoverability rather than as a literal mutual-information estimator. It is low when $R$ leaves the answer nearly as uncertain as $Q$ alone, and high when $R$ makes $A$ substantially more predictable under the scoring model.
The unnormalized numerator corresponds to per-token conditional bit gain under the scoring model, while the bounds in this appendix are stated in total mutual-information units. Thus the theory motivates the expected direction of $\mathcal{A}_{\mathrm{prob}}$, but the empirical tables validate the proxy behavior directly. Appendix~\ref{app:theory_validation} reports additional boundary diagnostics showing that the main conclusions are not driven by upper clipping of $\mathcal{A}_{\mathrm{prob}}$.

Under SSR's two-phase generation, the reasoning trace is generated as
\[
R \sim P(R \mid S, Q, A).
\]
The dependence between $A$ and $R$ admits the bound
\[
I(A; R \mid Q) \leq I(A; S \mid Q) + I(A; R \mid Q, S).
\]
This follows from the chain rule and non-negativity of mutual information:
\[
I(A;R\mid Q)
\leq I(A;S,R\mid Q)
= I(A;S\mid Q)+I(A;R\mid Q,S),
\]
where the inequality uses $I(A;S\mid Q,R)\geq 0$.
The first term is the skeleton-mediated channel bounded by Proposition~\ref{prop:skeleton_channel_bound}. The second term is residual anchoring during reasoning generation. In the worst case, it can be as large as $H(A\mid Q,S)\leq H(A\mid Q)$, which can make the overall bound vacuous when the skeleton provides little useful structure. Proposition~\ref{prop:skeleton_channel_bound} therefore gives a guarantee only on information transmitted through $S$; whether total probabilistic anchoring decreases under SSR is an empirical claim, supported by \cref{tab:anchoring_metrics,fig:mechanism_diagnosis,tab:anchoring_metrics_qwen3_8b}.

\subsection{Theory-to-Metric Validation and Residual Channel}
\label{app:theory_validation}

The skeleton-channel analysis above is intentionally a limited theorem: it constrains information that can pass through $S$, not all information that can pass through the final trace $R$. We therefore add lightweight diagnostics that connect the theory to the measured traces without turning the bound into a stronger claim than it supports.

\paragraph{Skeleton leakage proxy.}
The invariance assumption in Definition~\ref{def:func_inv} is stated in conditional mutual-information terms, which are difficult to estimate directly for free-form summaries. As an observable proxy, we measure question-filtered lexical leakage in the generated skeleton content summaries, after removing tags and \texttt{[HIGH]/[LOW]} markers. Skeleton summaries carry substantially less answer-specific lexical content than the realized reasoning trace produced from those skeletons (\cref{tab:theory_validation_diagnostics}). This does not prove a small $\sum_i \epsilon_i^{(n)}$, but it verifies the intended low-bandwidth operating regime: the skeleton is much closer to an operation plan than to a compressed answer rendering.

\paragraph{Residual-channel proxy.}
Because phase two still conditions on $A$, the residual term $I(A;R\mid Q,S)$ remains the main theoretical escape hatch. A direct residual-channel experiment would generate $R\sim P(R\mid Q,S)$ without answer access and compare it with full SSR. We do not report that stricter ablation here. Instead, we include an answer-swap sensitivity diagnostic: keeping $Q$ fixed and swapping the visible answer, we regenerate traces and measure how much the trace changes. SSR is less sensitive than NEU under word-level, content-level, and embedding-level changes, indicating that the final trace is less driven by the specific visible answer even after the skeleton has been realized. This is an operational residual-channel check, not a formal estimate of $I(A;R\mid Q,S)$.

\begin{table}[h]
\caption{Theory-to-metric diagnostics for SSR. Lower skeleton lexical leakage and lower answer-swap sensitivity indicate weaker answer transmission through the skeleton and final trace, respectively.}
\label{tab:theory_validation_diagnostics}
\centering
\begin{tabular}{lccc}
\toprule
\textbf{Diagnostic} & \textbf{NEU} & \textbf{SSR} & \textbf{SSR--NEU} \\
\midrule
Skeleton question-filtered lexical anchoring & -- & 4.694 & -- \\
Final SSR reason question-filtered lexical anchoring & -- & 11.352 & -- \\
Answer-swap ROUGE-L sensitivity & 0.812 & 0.799 & $-0.014$ \\
Answer-swap content-Jaccard sensitivity & 0.903 & 0.856 & $-0.047$ \\
Answer-swap embedding sensitivity & 0.0012 & 0.0007 & $-0.0005$ \\
\bottomrule
\end{tabular}
\end{table}

\paragraph{Probabilistic-boundary diagnostic.}
Since $\mathcal{A}_{\mathrm{prob}}$ clips normalized surprisal reduction into $[0,1]$, we also inspect the raw unnormalized endpoint bit gain $B(1.00)$ and clipping boundary rates (\cref{tab:aprob_boundary_diagnostics}). The upper clipping boundary is never reached for the main methods, so SSR's lower $\mathcal{A}_{\mathrm{prob}}$ is not an artifact of saturation. A small fraction of examples have negative endpoint bit gain, meaning the trace makes the answer less predictable under the scorer; this rate is also lower for SSR than for suppression baselines.

\begin{table}[h]
\caption{Probabilistic anchoring boundary diagnostic on the Qwen3-4B main methods. $B(1.00)$ is the unnormalized endpoint bit-gain diagnostic. $B(1.00)<0$ identifies examples whose trace increases answer surprisal, while $\mathcal{A}_{\mathrm{prob}}=1$ identifies upper-bound saturation.}
\label{tab:aprob_boundary_diagnostics}
\centering
\begin{tabular}{lcccc}
\toprule
\textbf{Method} & $\mathcal{A}_{\mathrm{prob}}$ & $B(1.00)$ mean & $B(1.00)<0$ & $\mathcal{A}_{\mathrm{prob}}=1$ \\
\midrule
NEU     & 5.829 & 14.446 & 4.9\% & 0.0\% \\
SUP     & 5.641 & 13.827 & 7.1\% & 0.0\% \\
AUG-SUP & 5.522 & 13.634 & 7.4\% & 0.0\% \\
SSR     & 4.606 & 11.283 & 3.7\% & 0.0\% \\
\bottomrule
\end{tabular}
\end{table}

\paragraph{Failure conditions.}
These diagnostics also clarify where the theory should not be overread. SSR is a mitigation mechanism, not a guarantee of faithful reasoning. If the skeleton is too coarse, it may fail to guide the trace and push answer dependence back into $I(A;R\mid Q,S)$. If the skeleton is too fine or includes answer-specific summaries, the skeleton channel itself can leak the response. If the correct reasoning structure is inherently answer-dependent, even tag choices and length can transmit information about $A$. The empirical claim is therefore that the final prompt operates in a useful middle regime for the evaluated tasks: skeletons are informative enough to guide generation, but abstract enough to reduce answer recoverability and answer-conditioned trajectory dependence.

\subsection{Mechanistic Hypothesis: Avoiding Ironic Process Amplification}

Section~\ref{sec:interp} interprets the behavior of suppression prompts through the lens of ironic process theory. We treat this as a mechanistic hypothesis rather than a formal proposition. Suppression prompts may increase answer salience because the model must avoid explicit answer leakage, whereas SSR shifts part of the generation objective toward matching a structural target $S$ rather than continuously checking whether each token reveals $A$. Since phase two still conditions on $A$, this hypothesis does not imply that SSR cannot attend to or use the answer; it only predicts a weaker need for explicit answer-exclusion monitoring.

\paragraph{Falsifiable predictions.}
The monitoring hypothesis predicts that, relative to NEU, suppression prompts (SUP/AUG-SUP) should exhibit (i) higher attention mass from reasoning tokens onto the answer span $A$, (ii) higher linear-probe accuracy for $A$-related features in mid-layer hidden states, and (iii) larger causal effects when answer-related activations are patched into the residual stream during reasoning generation. Conversely, SSR should reduce these answer-salience signatures relative to suppression while maintaining lower $\mathcal{A}_{\mathrm{prob}}$. We leave this mechanistic verification to future work.

\subsection{SSR Implementation Details}

The definitions of functional tags are given in \cref{tab:tag_definitions}. When designing the functional tag set, we draw insight from meta-reasoning paradigms \citep{zhang2025searching,sui2025meta} to describe a compact collection of reasoning operations.
In the final SSR prompt (Appendix~\ref{app:prompt}), each skeleton line follows the form \texttt{n. [TAG][HIGH/LOW] <short action sentence>}. The \texttt{[HIGH]} marker indicates steps where the reasoning trace should carry more local detail, such as difficult checks, central judgments, method choices, uncertainty boundaries, or high-information framing decisions. The prompt asks for 6--12 skeleton steps depending on task complexity, with diagnostic, importance-driven, or information-density-driven steps added only when warranted by the task.

\begin{table*}[t]
    \caption{Functional Tag Definitions.}
    \label{tab:tag_definitions}
    \centering
    \begin{tabular}{l l p{6.5cm}}
        \toprule
        \textbf{Tag} & \textbf{Full Name} & \textbf{Description} \\
        \midrule
        \texttt{PLAN} & Planning and Understanding & Comprehending input, defining goals/constraints, outlining a high-level plan. \\
        \texttt{RETR} & Retrieval & Searching for needed information from external or internal knowledge. \\
        \texttt{INFR} & Inference and Deduction & Logical reasoning, calculation, transformation, or generating intermediate conclusions. \\
        \texttt{EVAL} & Evaluation and Verification & Checking correctness, consistency, or sufficiency of prior results. \\
        \texttt{SUMM} & Summary and Refinement & Integrating intermediate results, refining expression, producing final answers. \\
        \texttt{BTRK} & Backtrack & When evaluation fails, returning to earlier decisions to revise strategy. \\
        \texttt{RFLX} & Reflection & Reviewing the reasoning to derive insights or generate new plans/backtracks. \\
        \texttt{BRCH} & Branch & Considering multiple possible reasoning paths and selecting one. \\
        \bottomrule
    \end{tabular}
    \vskip -0.1in
\end{table*}

\textbf{Content Summary Guidelines.} Skeleton lines should:
\begin{itemize}
    \item Describe the operation, check, selection criterion, or response-shaping action rather than the result.
    \item Stay short, task-specific, and free of template placeholders.
    \item Avoid revealing concrete outputs, intermediate values, or final-answer content.
    \item Keep one primary reasoning intent per line; split composite operations into separate steps.
\end{itemize}

\subsection{Structural Properties of SSR Skeletons}
\label{app:skeleton_structure}

We analyze the SSR skeletons generated during data construction to verify that they are organized structural guides rather than fixed prompt templates. The skeleton lengths span a broad range, and tag usage is clearly non-uniform: \texttt{INFR} is the dominant operation, with \texttt{PLAN} and \texttt{EVAL} also frequent, while \texttt{BRCH} and \texttt{BTRK} appear less often (\cref{fig:skeleton_step_tag}). This distribution indicates a broad but task-adaptive scaffold, not a rigid hand-written pattern.

\begin{figure*}[t]
\centering
\includegraphics[width=0.82\textwidth]{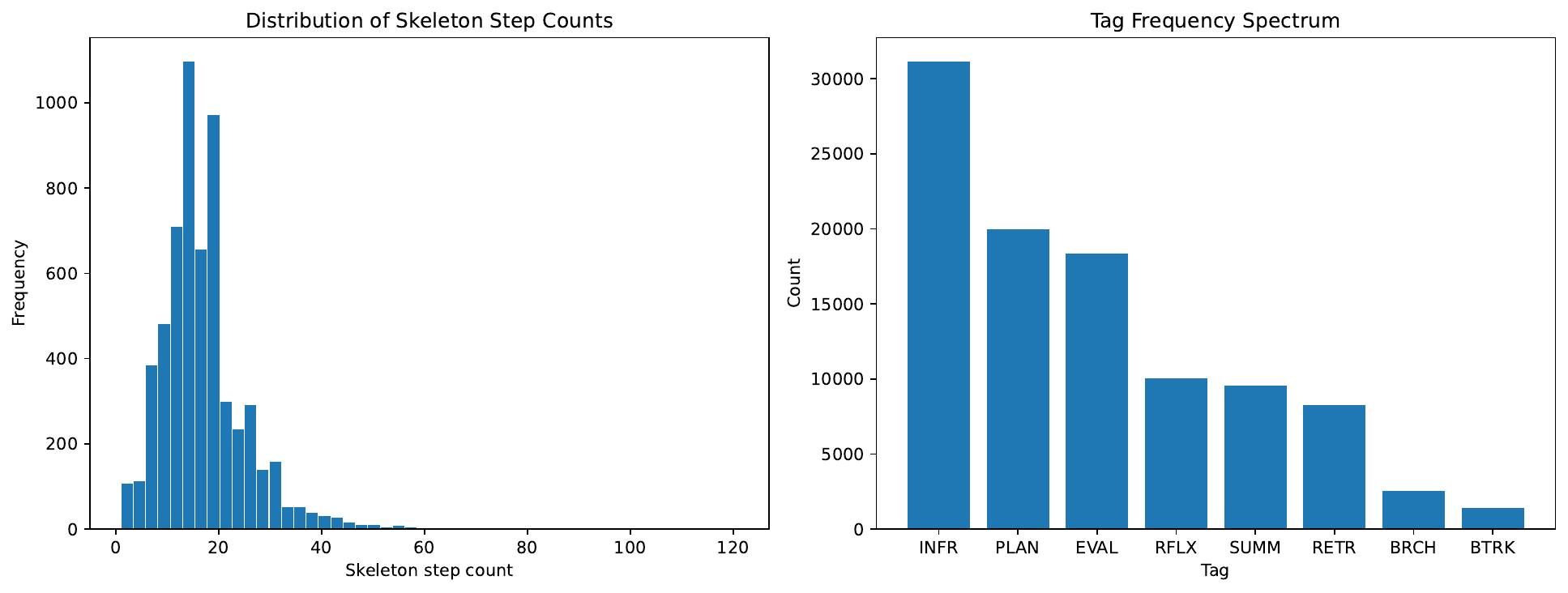}
\caption{Skeleton step-count distribution and functional-tag frequency. SSR skeletons vary substantially in length and show non-uniform tag usage, with inference, planning, and evaluation serving as the most common operations.}
\label{fig:skeleton_step_tag}
\end{figure*}

The positional and transition patterns further show coherent staged organization. \texttt{PLAN} and \texttt{RETR} tend to appear earlier, \texttt{EVAL} and especially \texttt{SUMM} later, while \texttt{INFR}, \texttt{BRCH}, \texttt{RFLX}, and \texttt{BTRK} concentrate more in the middle-to-late stages (\cref{fig:skeleton_position_transition}). The transition matrix reveals strong \texttt{RETR}$\rightarrow$\texttt{INFR}, \texttt{INFR}$\rightarrow$\texttt{INFR}, \texttt{INFR}$\rightarrow$\texttt{EVAL}, and persistent late \texttt{SUMM} behavior, matching a natural retrieve, infer, verify, and summarize progression.

\begin{figure*}[t]
\centering
\includegraphics[width=0.86\textwidth]{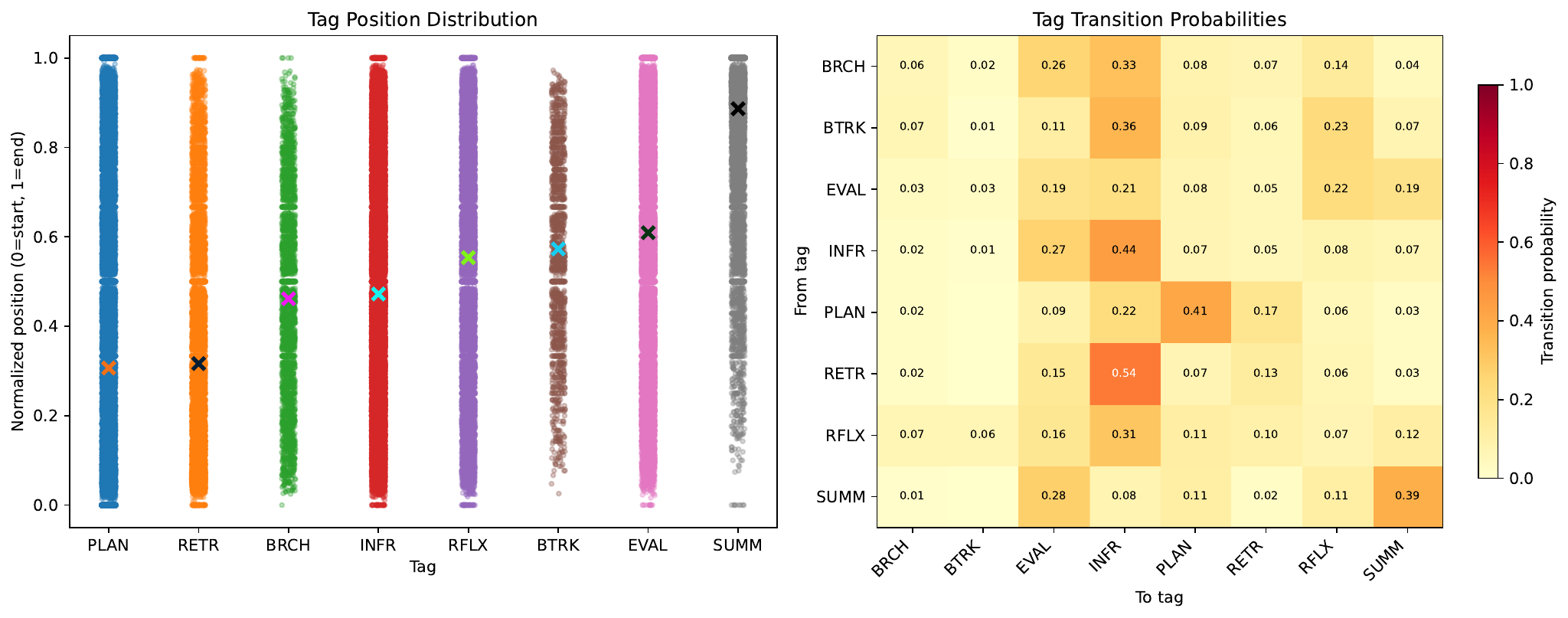}
\caption{Functional-tag positions and transition probabilities. SSR skeletons exhibit stage-specific tag placement and nontrivial sequential regularities, supporting the claim that the scaffold is structured but non-rigid.}
\label{fig:skeleton_position_transition}
\end{figure*}

We also compare first-round skeletons with refined output skeletons. Refined skeletons preserve the same broad positional ordering while adjusting step counts and tag locations (\cref{fig:skeleton_round_compare}), suggesting that the second pass refines structure rather than collapsing to a single template.

\begin{figure*}[t]
\centering
\includegraphics[width=0.86\textwidth]{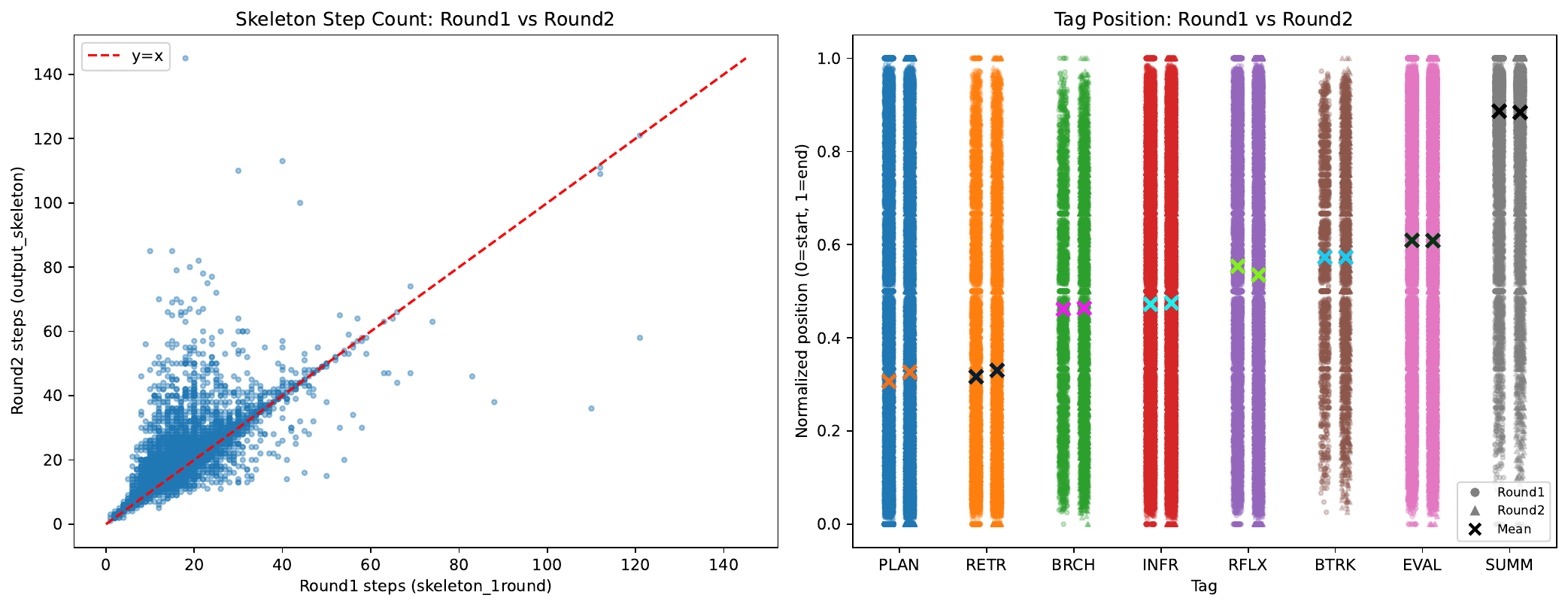}
\caption{Round-1 versus refined skeleton comparison. The refined skeletons preserve the broad structure of the first-pass plans while adjusting step counts and tag positions.}
\label{fig:skeleton_round_compare}
\end{figure*}

\section{Data and Evaluation Setup}
\label{app:data_eval_setup}

\subsection{Reference Answer Construction}
\label{app:data_construction}

We construct our experiment dataset using a self-improvement pipeline to simulate the situations that the high-quality responses are accessible but human-written or gold intermediate reasoning traces are lacking.

\paragraph{Data source.}
We sample user queries from the LMArena human preference corpus (140k conversations) hosted at
\url{https://huggingface.co/datasets/lmarena-ai/arena-human-preference-140k}
and introduced in prior work on the text arena \citep{chiang2024chatbot}.

\paragraph{Overview.}
For each sampled query, we use Qwen3-Max \citep{qwen3} to construct a high-quality reference answer via a multi-stage, iterative refinement pipeline. The procedure alternates between (i) generating diverse candidate responses, (ii) self-evaluating candidates along key quality dimensions, and (iii) aggregating the strongest components into improved candidates, repeating this process for multiple refinement loops.

\paragraph{Pipeline stages.}
Let $N$ denote the number of independent rollouts produced per query, $K$ the number of candidate slots maintained after aggregation, $M$ the number of evaluated candidates sampled for synthesizing each slot, and $T$ the number of refinement loops.
\begin{enumerate}[leftmargin=*, itemsep=2pt, topsep=2pt]
    \item \textbf{Candidate generation.} We prompt Qwen3-Max to produce $N$ independent response rollouts for the same query to encourage diversity in reasoning paths and content coverage.
    \item \textbf{Candidate evaluation.} The model then self-evaluates each rollout, assigning a scalar score based on accuracy, coherence, and completeness. These scores are used for quality reference during the aggregation step.
    \item \textbf{Candidate aggregation.} We construct a new set of $K$ candidates. For each slot, we randomly sample $M$ evaluated candidates and synthesize an improved response by combining their highest-scoring components (e.g., correct facts, clearer explanations, or more complete coverage).
    \item \textbf{Improvement loop.} The newly synthesized $K$ candidates are fed back into the evaluation and aggregation stages. We repeat this improvement cycle for $T$ loops, yielding a final, polished reference answer at convergence.
\end{enumerate}

\paragraph{Algorithmic description.}
Algorithm~\ref{alg:self_refine} summarizes the refinement procedure used to produce the definitive reference answer per query.
\begin{algorithm}[tb]
\caption{Iterative self-refinement for reference answer construction}
\label{alg:self_refine}
\begin{algorithmic}[1]
\REQUIRE Query $x$; model $\mathcal{M}$ (Qwen3-Max); rollouts $N$; slots $K$; sample size $M$; loops $T$
\STATE $\mathcal{C} \leftarrow \{\mathcal{M}(x)\}_{i=1}^{N}$ \hfill \COMMENT{Generate $N$ independent candidates}
\FOR{$t = 1, \ldots, T$}
    \STATE $\mathcal{S} \leftarrow \{\textsc{Score}(c)\;|\;c \in \mathcal{C}\}$ \hfill \COMMENT{Self-evaluate on accuracy/coherence/completeness}
    \STATE $\mathcal{C}' \leftarrow \emptyset$
    \FOR{$j = 1, \ldots, K$}
        \STATE Sample $\{c_{j,1},\dots,c_{j,M}\}$ from $\mathcal{C}$ (optionally biased by $\mathcal{S}$)
        \STATE $\mathcal{S}_{j} \leftarrow \{\textsc{Score}(c_{j,m})\}_{m=1}^{M}$ \hfill \COMMENT{Use only scores of sampled candidates}
        \STATE $c'_j \leftarrow \textsc{Synthesize}(\{c_{j,1},\dots,c_{j,M}\}, \mathcal{S}_{j})$
        \STATE $\mathcal{C}' \leftarrow \mathcal{C}' \cup \{c'_j\}$
    \ENDFOR
    \STATE $\mathcal{C} \leftarrow \mathcal{C}'$ \hfill \COMMENT{Update candidates for next loop}
\ENDFOR
\STATE \textbf{return} $\arg\max_{c \in \mathcal{C}} \textsc{Score}(c)$ \hfill \COMMENT{Final reference answer}
\end{algorithmic}
\end{algorithm}

\paragraph{Rationale.}
Our construction yields reference responses that reflect the outcome of substantial deliberation and iterative error-correction, while keeping intermediate reasoning implicit. As a result, the final answer provides a strong target without exposing step-by-step traces that could be trivially copied. This design matches realistic deployment conditions, where systems are typically evaluated on inputs and final outputs rather than on access to the full internal decision process. It also creates a controlled setting in which a model must justify or explain an already-produced outcome using only the surface form of the response, which is exactly where post-hoc rationalization is most likely to arise and most difficult to detect.

\subsection{Evaluation Tasks}
\label{app:benchmarks}

\subsubsection{In-domain open-ended reasoning benchmarks}

\begin{description}
    \item[ArenaHard-v2.0 (Human Preference)] \hfill \\
    An automatic evaluation tool for instruction-tuned LLMs designed to simulate the ``Chatbot Arena'' environment. It boasts the highest correlation and separability to human-preference benchmarks (LMArena) among popular open-ended benchmarks. It assesses the model's ability to handle complex, open-ended inquiries using automatic judges (e.g., GPT-4, Gemini) as approximators for human preference. \\
    \textbf{Statistics:} The V2.0 dataset contains 500 fresh, challenging real-world user queries covering topics like software engineering and mathematics, alongside 250 creative writing queries sourced from Chatbot Arena. \\
    \textbf{Source:} \url{https://github.com/lmarena/arena-hard-auto}

    \item[EQ-Bench 3 (Emotional Intelligence)] \hfill \\
    A multi-turn benchmark assessing active emotional intelligence skills, including empathy, social dexterity, psychological insight, and analytical depth. Unlike knowledge-based tests, it places models in role-play scenarios (e.g., conflict mediation, relationship drama) or analysis tasks to test their ability to reason about human emotions. \\
    \textbf{Statistics:} Evaluation utilizes two primary methods: Rubric Scoring, where a judge model (default: Claude Sonnet 3.7) assigns a multi-criteria score from $0$ to $100$, and Pairwise ELO Analysis, which ranks models via head-to-head comparisons. \\
    \textbf{Source:} \url{https://github.com/EQ-bench/eqbench3}

    \item[IFEval (Instruction Following)] \hfill \\
    IFEval evaluates instruction-following using programmatically verifiable constraints (e.g., required keywords, length constraints, formatting rules), enabling reproducible automatic checking. \\
    \textbf{Statistics:} Contains 541 prompts across 25 distinct instruction types. \\
    \textbf{Source:} \url{https://github.com/google-research/google-research/tree/master/instruction_following_eval}
    
    \item[MultiChallenge (Multi-turn conversations)] \hfill \\
    A benchmark designed to evaluate multi-turn instruction-following capabilities. It tests whether models can maintain constraints, recall information, and edit responses over the course of a long conversation, covering four specific categories: Inference Memory, Instruction Retention, Reliable Version Editing, and Self-Coherence. \\
    \textbf{Statistics:} The dataset consists of 273 test conversations with an average of 5 turns and 1231.7 words per conversation. The breakdown is as follows:
    \begin{itemize}
        \item Inference Memory: 113 conversations
        \item Instruction Retention: 69 conversations
        \item Reliable Version Editing: 41 conversations
        \item Self-Coherence: 50 conversations
    \end{itemize} 
    \textbf{Source:} \url{https://github.com/ekwinox117/multi-challenge}

\end{description}

\subsubsection{Out-of-domain (OOD) reasoning benchmarks}

\begin{description}
    \item[GPQA-D (Science)] \hfill \\
    A graduate-level, multiple-choice QA benchmark written by domain experts in biology, chemistry, and physics. We use the ``Diamond'' subset, designed to be ``Google-proof'' and challenging even for experts. \\
    \textbf{Statistics:} 198 questions utilized for OOD scientific evaluation. \\
    \textbf{Source:} \url{https://huggingface.co/datasets/Idavidrein/gpqa}

    \item[AIME 2025 (Mathematics)] \hfill \\
    Tests Olympiad-style mathematical reasoning using problems from the 2025 American Invitational Mathematics Examination (AIME I \& II). \\
    \textbf{Statistics:} 30 problems requiring exact-match integer answers. \\
    \textbf{Source:} \url{https://huggingface.co/datasets/math-ai/aime25}
\end{description}

\section{Training and Distillation Details}
\label{app:training}

\subsection{Training Hyperparameters}

We fine-tune the series of Qwen3 and NBG4 models using OpenRLHF \citep{hu2024openrlhf}, and the hyperparameters are presented in \cref{tab:train_params}. Other model sizes follow the same training recipe unless constrained by memory, in which case we adjust tensor/data parallelism while keeping the effective batch size and optimization settings unchanged.

\begin{table}[t]
\caption{Training Hyperparameters for Qwen3-8B}
\label{tab:train_params}
\centering
\begin{tabular}{lc}
\toprule
\textbf{Hyperparameter} & \textbf{Value} \\
\midrule
Base Model & Qwen3-8B \\
Max Sequence Length & 32,768 \\
Global Batch Size & 256 \\
Micro Batch Size & 1 \\
Learning Rate & $4 \times 10^{-5}$ \\
LR Warmup Ratio & 0.01 \\
Weight Decay (L2) & 0.01 \\
Max Epochs & 4 \\
Optimizer Strategy & ZeRO Stage 3 \\
Precision & BF16 \\
Ring Attention Size & 4 \\
Ring Head Stride & 4 \\
Gradient Checkpointing & Enabled \\
Sample Packing & Enabled \\
Dataset Size & 100,000 \\
\bottomrule
\end{tabular}
\vskip -0.1in
\end{table}

\subsection{SSR-D: Details of Teacher Training and Student Distillation}
\label{app:train_teacher}

SSR-D uses a two-stage teacher-student pipeline to internalize the SSR format. We first construct SSR seed examples and train an SSR-format teacher model. The fine-tuned teacher then generates the final distilled traces used to train the target student models.

For each query-answer pair $(Q,A)$, an SSR trace consists of a structural skeleton $S^*$ and a full reverse-CoT reasoning trace $R^*$. The student is trained with two supervised objectives:
\begin{enumerate}[leftmargin=*, itemsep=2pt]
    \item \textbf{Skeleton Generation:} The model first predicts the teacher skeleton from the query and answer:
    \[
        \mathcal{L}_S = -\log p_\theta(S^* \mid Q, A)
    \]
    
    \item \textbf{Reasoning Reconstruction:} The model then reconstructs the teacher reasoning trace conditioned on the skeleton:
    \[
        \mathcal{L}_R = -\log p_\theta(R^* \mid Q, A, S^*)
    \]
\end{enumerate}

The combined objective:
\[
    \mathcal{L} = \mathcal{L}_S + \mathcal{L}_R
\]
trains the student to preserve the SSR ordering, first planning through a skeleton and then realizing the full reasoning trace from that skeleton.

The teacher pipeline proceeds in three steps. First, Qwen3-Max produces temporary reference reasoning traces for the sampled $(Q,A)$ pairs. These traces are used only to construct the SSR teacher and are not the final traces used for student distillation. Second, Qwen3-235B-Instruct-2507 segments each trace into logical steps, assigns functional tags from our restricted vocabulary, and rewrites each step into a content-invariant skeleton summary. This produces seed tuples $(Q,A,S^*,R^*)$; the skeleton-generation prompt is provided in Appendix~\ref{app:prompt}. Third, we fine-tune Qwen3-235B-Instruct-2507 with OpenRLHF on these SSR seed tuples to obtain the SSR teacher.

The fine-tuned SSR teacher then generates 100k final SSR-D traces on another 100k query-answer pairs for student training. These teacher-generated traces, rather than the raw seed annotations alone, are used to distill the target models evaluated as SSR-D.



\section{Additional Experimental Results}
\label{app:additional_results}

\subsection{Ablation Studies}
\label{app:ablation}

We conduct systematic ablations to quantify the contribution of each SSR component and understand scaling behavior.

\subsubsection{Component Ablation}
\label{sec:component_ablation}

\begin{table}[t]
    \caption{Detailed component ablation on Qwen3-8B.}
    \label{tab:component_ablation}
    \centering
    \begin{tabular}{lccc}
        \toprule
        \textbf{Configuration} & \textbf{ArenaHard} & \textbf{GPQA-D} \\
        \midrule
        SSR-D             & 59.5 & 56.6 \\
        $-$ two-phase generation  & 52.2 & 51.6 \\
        $-$ Functional tags       & 55.8 & 54.2 \\
        $-$ Content skeletons     & 53.4 & 52.8 \\
        NEU Baseline              & 50.8 & 49.5 \\
        \bottomrule
    \end{tabular}
    \vskip -0.1in
\end{table}

We ablate the main components of SSR in \cref{tab:component_ablation}.

\paragraph{Two-phase Generation.}
Removing the two-phase protocol (skeleton $\rightarrow$ reasoning) and instead generating skeletons interleaved with reasoning produces the largest degradation ($-7.0$ ArenaHard). This confirms that explicit separation of structural planning from reasoning execution is essential; the skeleton must be complete before reasoning begins to provide effective guidance and satisfy the functional invariance property (Definition~\ref{def:func_inv}).

\paragraph{Functional Tags vs.\ Content Skeletons.}
Both components contribute substantially, but content skeletons have a larger impact on task performance ($-5.8$ ArenaHard when removed vs.\ $-3.4$ for tags). This suggests complementary roles: tags provide coarse structural scaffolding that constrains reasoning topology, while content skeletons provide fine-grained guidance that improves reasoning quality.

\subsubsection{Teacher Model Scaling}
\label{sec:teacher_scaling}

We examine how the fine-tuned SSR teacher affects student performance for distilled SSR (SSR-D).

\begin{table}[t]
    \caption{Effects of fine-tuned SSR teacher choice on student performance (Student: Qwen3-8B).}
    \label{tab:teacher_scaling}
    \centering
    \begin{tabular}{lccc}
        \toprule
        \textbf{Teacher} & \textbf{ArenaHard} & \textbf{GPQA-D} & \textbf{AIME '25} \\
        \midrule
        Qwen3-32B   & 57.7 & 54.8 & 42.3\\
        \makecell[l]{Qwen3-235B-\\Instruct-2507}  & \textbf{59.5} & 56.6 & 44.2 \\
        NBG-3.5-Pro\footnotemark  & 58.9 & \textbf{57.8} & \textbf{45.0} \\
        \bottomrule
    \end{tabular}
    \vskip -0.1in
\end{table}
\footnotetext{\url{https://www.nanbeige.com/portal}}

The fine-tuned Qwen3-235B-Instruct-2507 SSR teacher is our default generator for SSR-D traces. Teacher choice affects the downstream profile (\cref{tab:teacher_scaling}): Qwen3-235B-Instruct-2507 yields the best ArenaHard result, while NBG-3.5-Pro gives stronger GPQA-D and AIME 2025 results. Both stronger teachers outperform the Qwen3-32B teacher, indicating that SSR-D benefits from higher-quality SSR-format supervision.

\subsubsection{Structured Rendering Ablation}
\label{sec:json_naturalness_ablation}

We also ablate how structure is imposed at generation time. A schema-rendered route asks the model to emit typed step fields, validates them, and deterministically renders them into the same \texttt{<skeleton>} and \texttt{<reason>} interface. This removes many surface-format failures, but it also changes the target distribution. In contrast, the final SSR prompt uses natural-language structural guidance: the model is asked to produce a structured skeleton and a derivational reasoning trace directly, so structure is induced without forcing each paragraph through a schema field.

\begin{table}[t]
    \caption{Ablation of schema-rendered versus natural-language structural guidance. Anchoring columns use the reference-normalized display scale from Appendix~\ref{app:display_score_normalization}; lower is better for anchoring metrics and template rigidity.}
    \label{tab:json_naturalness_ablation}
    \centering
    \begin{tabular}{lrrrrrrr}
        \toprule
        \textbf{Variant} & $\mathcal{A}_{\mathrm{lex}}^{\mathrm{disp}}$ & $\mathcal{A}_{\mathrm{traj}}^{\mathrm{disp}}$ & $\mathcal{A}_{\mathrm{prob}}^{\mathrm{disp}}$ & \textbf{Avg. Tokens} & \textbf{Paragraphs} & \textbf{One-Sent. Para.} & \textbf{Rigidity} \\
        \midrule
        Schema-rendered SSR & 26.5 & 36.5 & 40.0 & 294.1 & 8.83 & 54.2\% & 1.170 \\
        Natural-language SSR & 35.6 & 31.7 & 37.0 & 331.9 & 5.10 & 2.3\% & 0.000 \\
        Natural-language control & 34.9 & 30.8 & 37.6 & 334.0 & 5.10 & 2.5\% & -0.003 \\
        \bottomrule
    \end{tabular}
    \vskip -0.1in
\end{table}

The schema-rendered variant improves lexical anchoring, but it worsens both trajectory anchoring and probabilistic anchoring relative to the natural-language structured baselines (\cref{tab:json_naturalness_ablation}). More importantly, the rendered traces become visibly template-like: they contain many more paragraphs (8.83 vs.\ 5.10) and a much larger fraction of one-sentence paragraphs (54.2\% vs.\ about 2--3\%). This indicates that parser-level structure is not equivalent to natural RCoT structure. The renderer removes schema syntax, but it cannot remove the field-aligned generation habit induced by the schema: each local field tends to justify a single action rather than participate in a continuous derivational trace. We therefore use natural-language structural guidance in SSR. It preserves the intended skeleton pressure while avoiding a rigid step-justification style that would be a poor target for downstream finetuning.

\subsubsection{Dense Skeleton-Realization Ablation}
\label{sec:indexed_dense_diagnostic}

We further ablate how tightly the final trace should realize the skeleton. A dense realization variant asks each reasoning paragraph to correspond to a skeleton step and encourages each paragraph to develop a local uncertainty, criterion, evidence need, or constraint before narrowing to the next task move. It also favors task-object wording in closing steps, so checks and refinements are attached to concrete objects such as facts, calculations, code paths, recommendations, boundaries, or artifacts rather than generic references to the answer or response.

\begin{table}[t]
    \caption{Core SSR versus dense skeleton realization on Qwen3-4B. Anchoring columns use the reference-normalized display scale from Appendix~\ref{app:display_score_normalization}; lower is better. \textbf{Simple$>$300} is the share of simple-surface tasks whose traces exceed 300 tokens.}
    \label{tab:indexed_dense_diagnostic}
    \centering
    \begin{tabular}{lrrrrrrr}
        \toprule
        \textbf{Variant} & $\mathcal{A}_{\mathrm{lex}}^{d}$ $\downarrow$ & $\mathcal{A}_{\mathrm{traj}}^{d}$ $\downarrow$ & $\mathcal{A}_{\mathrm{prob}}^{d}$ $\downarrow$ & \textbf{Tok.} & \textbf{Endpt.} $\uparrow$ & \textbf{Lang.} $\uparrow$ & \textbf{Simple$>$300} $\downarrow$ \\
        \midrule
        Core SSR & 30.3 & 33.3 & 40.8 & 356.3 & 0.989 & 0.713 & 45.3\% \\
        Dense & 28.9 & 27.8 & 44.1 & 349.2 & 0.998 & 0.688 & 67.2\% \\
        \bottomrule
    \end{tabular}
    \vskip -0.1in
\end{table}

Dense realization addresses several quality issues that the core prompt does not fully eliminate. First, the explicit skeleton-to-paragraph mapping makes the skeleton less likely to become a detached outline, since each line has a corresponding local reasoning paragraph. Second, the denser paragraph rule tends to make traces more exploratory and task-internal rather than compressed summaries of the final answer. Third, object-centered wording reduces detached meta-comments about whether a ``response'' or ``answer'' is aligned, and pushes checks toward the actual task object. These changes improve two anchoring diagnostics and the abstract endpoint guardrail: lexical display anchoring drops from 30.3 to 28.9, trajectory display anchoring drops from 33.3 to 27.8, and abstract endpoint alignment rises from 0.989 to 0.998 (\cref{tab:indexed_dense_diagnostic}).

The same changes introduce trade-offs. Dense realization raises probabilistic display anchoring from 40.8 to 44.1, suggesting that the more concrete and endpoint-stable traces carry more answer-recoverable information. It also increases the share of simple-surface tasks with long traces (45.3\% to 67.2\%) and slightly worsens language matching (0.713 to 0.688), mostly by drifting non-English prompts toward English. We therefore treat dense realization as a quality-oriented ablation rather than the main SSR method: it clarifies which remaining failure modes can be improved by tighter skeleton realization, while also showing that additional endpoint and density constraints can reintroduce answer information and distributional artifacts.

\subsection{Qualitative Case Study}
\label{app:case_study}

We further inspect a two-step conversation: the first turn asks why teenage girls become fixated on women footballers' private lives, and the follow-up narrows the question to ``shipping.'' As shown in \cref{fig:case_study_boxes}, SSR and SSR-D both receive 5.0 on self-contained derivation, compared with 3.5 for SUP and 3.0 for NEU. The most diagnostic fragments are shown below:

\begin{figure*}[t]
\centering
\begin{minipage}[t]{0.46\textwidth}
\begin{tcolorbox}[
    enhanced,
    title={NEU \hfill Self-contained derivation: 3.0/5},
    colback=Mycolor1!38,
    colframe=Mycolor1!85!black,
    coltitle=black,
    fonttitle=\bfseries\small,
    opacityback=0.62,
    boxrule=0.7pt,
    arc=4mm,
    height=2.1cm,
    valign=center,
    left=2mm,right=2mm,top=1.5mm,bottom=1.5mm,
    drop fuzzy shadow,
]
``Identify Key Themes from the Provided Response ... This thought process aligns perfectly with the structure and content of the assistant's provided response''
\end{tcolorbox}
\end{minipage}
\hfill
\begin{minipage}[t]{0.46\textwidth}
\begin{tcolorbox}[
    enhanced,
    title={SUP \hfill Self-contained derivation: 3.5/5},
    colback=yellow!36!Mycolor1,
    colframe=orange!75!black,
    coltitle=black,
    fonttitle=\bfseries\small,
    opacityback=0.62,
    boxrule=0.7pt,
    arc=4mm,
    height=2.1cm,
    valign=center,
    left=2mm,right=2mm,top=1.5mm,bottom=1.5mm,
    drop fuzzy shadow,
]
``Determine the underlying motivation ... the key differentiator is the high visibility of LGBTQ+ relationships ...''
\end{tcolorbox}
\end{minipage}

\vspace{2mm}

\begin{minipage}[t]{0.46\textwidth}
\begin{tcolorbox}[
    enhanced,
    title={SSR \hfill Self-contained derivation: 5.0/5},
    colback=Mycolor3!42,
    colframe=Mycolor3!70!black,
    coltitle=black,
    fonttitle=\bfseries\small,
    opacityback=0.62,
    boxrule=0.7pt,
    arc=4mm,
    height=2.1cm,
    valign=center,
    left=2mm,right=2mm,top=1.5mm,bottom=1.5mm,
    drop fuzzy shadow,
]
``deeper dive into a specific behavior ... moving from general fascination to the specific mechanics of romantic projection''
\end{tcolorbox}
\end{minipage}
\hfill
\begin{minipage}[t]{0.46\textwidth}
\begin{tcolorbox}[
    enhanced,
    title={SSR-D \hfill Self-contained derivation: 5.0/5},
    colback=Mycolor2!34!Mycolor3,
    colframe=Mycolor2!70!black,
    coltitle=black,
    fonttitle=\bfseries\small,
    opacityback=0.62,
    boxrule=0.7pt,
    arc=4mm,
    height=2.1cm,
    valign=center,
    left=2mm,right=2mm,top=1.5mm,bottom=1.5mm,
    drop fuzzy shadow,
]
``I need to recall the previous context ... Now, the user is zooming in on the shipping aspect specifically''
\end{tcolorbox}
\end{minipage}
\caption{Qualitative comparison of self-contained derivation scores. Box color shifts from lower-scoring red/orange traces to higher-scoring green/blue traces.}
\label{fig:case_study_boxes}
\end{figure*}

The SSR and SSR-D traces read as context-grounded derivations from the prior conversation, while SUP and NEU more directly organize the explanation around preselected conclusions.

\section{Prompt Templates}
\label{app:prompt}

For the suppression baselines (\textbf{NEU, SUP, AUG-SUP}), we adapt the abductive reasoning prompt from \citet{sakana2025rlt} by incorporating semantic suppression constraints (highlighted in bold type). This ensures the model retains the capability to generate relatively high-quality reasoning traces.

\tcbset{
    promptstyle/.style={
        breakable,
        enhanced,
        colback=white,
        colframe=black!70,
        coltitle=white,
        colbacktitle=black!70,
        fonttitle=\bfseries\ttfamily,
        fontupper=\small\ttfamily,
        boxrule=1pt,
        arc=2mm,
        left=3mm, right=3mm, top=3mm, bottom=3mm,
        title after break={ (Continued)}
    }
}

\begin{tcolorbox}[promptstyle, title=Neutral RCG Prompt (NEU)]
Your role as an assistant involves providing precise and accurate solutions before providing detailed explanations with your full work showing your systematic thinking process leading to each solution. Your explanations should show how you engaged in a comprehensive cycle of analysis, summarizing, exploration, reassessment, reflection, backtracing, and iteration to develop well-considered thinking process. Please structure your response into two main sections: Solution and Explanation. In the Solution section, present your well-thought solution that accurately answers the question. The solution should remain a logical, accurate, concise expression style and detail necessary step needed to reach the conclusion, formatted as follows: \textless|begin\_of\_solution|\textgreater\ \{final formatted, precise, and clear solution\} \textless|end\_of\_solution|\textgreater. In the Explanation section, comprehensively detail your reasoning process from the question to your solutions using the specified format: \textless|begin\_of\_explanation|\textgreater\ \{explanation with steps separated with '\textbackslash n\textbackslash n'\} \textless|end\_of\_explanation|\textgreater\ Each step should show logical connections and detailed considerations leading to your solutions such as analyzing questions, summarizing relevant findings, brainstorming new ideas, verifying the accuracy of the current steps, refining any errors, and revisiting previous steps.
\end{tcolorbox}

\vskip 0.1in

\begin{tcolorbox}[promptstyle, title=Suppression Prompt (SUP)]
Your role as an assistant involves providing precise and accurate solutions before providing detailed explanations with your full work showing your systematic thinking process leading to each solution. Your explanations should show how you engaged in a comprehensive cycle of analysis, summarizing, exploration, reassessment, reflection, backtracing, and iteration to develop well-considered thinking process. Please structure your response into two main sections: Solution and Explanation. In the Solution section, present your well-thought solution that accurately answers the question. The solution should remain a logical, accurate, concise expression style and detail necessary step needed to reach the conclusion, formatted as follows: \textless|begin\_of\_solution|\textgreater\ \{final formatted, precise, and clear solution\} \textless|end\_of\_solution|\textgreater. In the Explanation section, comprehensively detail your reasoning process using the specified format: \textless|begin\_of\_explanation|\textgreater\ \{explanation with steps separated with '\textbackslash n\textbackslash n'\} \textless|end\_of\_explanation|\textgreater\ Each step should show detailed considerations leading to your solutions such as analyzing questions, summarizing relevant findings, brainstorming new ideas, verifying the accuracy of the current steps, refining any errors, and revisiting previous steps. \textbf{**DO NOT** explicitly output or hint at any information of Solution section in the Explanation section. **DO NOT** explicitly output or hint at any information of Solution section in the Explanation section. **DO NOT** explicitly output or hint at any information of Solution section in the Explanation section.}
\end{tcolorbox}


\begin{tcolorbox}[promptstyle, title=Augmented Suppression Prompt (AUG-SUP)]
Your role as an assistant involves providing precise and accurate solutions before providing detailed explanations with your full work showing your systematic thinking process leading to each solution. Your explanations should show how you engaged in a comprehensive cycle of analysis, summarizing, exploration, reassessment, reflection, backtracing, and iteration to develop well-considered thinking process. Please structure your response into two main sections: Solution and Explanation. In the Solution section, present your well-thought solution that accurately answers the question. The solution should remain a logical, accurate, concise expression style and detail necessary step needed to reach the conclusion, formatted as follows: \textless|begin\_of\_solution|\textgreater\ \{final formatted, precise, and clear solution\} \textless|end\_of\_solution|\textgreater. In the Explanation section, comprehensively detail your reasoning process using the specified format: \textless|begin\_of\_explanation|\textgreater\ \{explanation with steps separated with '\textbackslash n\textbackslash n'\} \textless|end\_of\_explanation|\textgreater\ Each step should show detailed considerations leading to your solutions such as analyzing questions, summarizing relevant findings, brainstorming new ideas, verifying the accuracy of the current steps, refining any errors, and revisiting previous steps. \textbf{**DO NOT** explicitly output any information of Solution section in the Explanation section. **DO NOT** explicitly output any information of Solution section in the Explanation section. **DO NOT** explicitly output any information of Solution section in the Explanation section. **PROHIBITION**: When outputting your Explanation, you are strictly forbidden from displaying ANY discernible signs that you have peeked at the Solution. **DO NOT** explicitly output any information of Solution section in the Explanation section. Should ANY form of Solution leakage occur, you will be severely punished by the Almighty Ruler.}
\end{tcolorbox}


\begin{tcolorbox}[promptstyle, title=Structural Skeleton-guided Reasoning Prompt (SSR)]
You are an expert AI assistant reconstructing a structured reasoning trace for a final assistant turn in a dialogue.

Output exactly two blocks, in this order: \textasciigrave{}\textless{}skeleton\textgreater{}\textasciigrave{} and \textasciigrave{}\textless{}reason\textgreater{}\textasciigrave{}. The already-provided assistant turn is context; return the reasoning trace only.

After closing \textasciigrave{}\textless{}/skeleton\textgreater{}\textasciigrave{}, the very next nonblank line must be \textasciigrave{}\textless{}reason\textgreater{}\textasciigrave{}. Never write \textasciigrave{}\textless{}/reason\textgreater{}\textasciigrave{} before writing \textasciigrave{}\textless{}reason\textgreater{}\textasciigrave{}.

The trace should be moderately developed: more informative than a terse outline, less expansive than a long chain of thought. Add depth where the task or the importance of a judgment calls for it, while keeping the trace derivational rather than a surface restatement. Use deeper reflection dynamically for objectively hard steps and for subjectively central steps where the response value depends on a key conclusion, audience need, method choice, uncertainty boundary, implementation tradeoff, or paper-style claim.

\#\#\# Tags

Use these tags only:

- [PLAN]: understand the request, constraints, and answer shape.
- [RETR]: bring in needed facts, context, or remembered information.
- [INFR]: infer, calculate, transform, or connect task details.
- [EVAL]: check a risky assumption, calculation, source, or consistency point.
- [SUMM]: organize the constraints, checks, and response plan.
- [BRCH]: compare plausible approaches before selecting one.
- [RFLX]: pause on an important judgment and name what keeps it disciplined.
- [BTRK]: revise a concrete earlier path after a check fails.

\#\#\# Skeleton Rules

1. Put one numbered line per step inside \textasciigrave{}\textless{}skeleton\textgreater{}\textasciigrave{}.
2. After the number, write one real tag from the tag list, then \textasciigrave{}[HIGH]\textasciigrave{} or \textasciigrave{}[LOW]\textasciigrave{} with no space between them, then a short task-specific action sentence.
3. Use 6-8 steps for simple tasks, 8-10 for ordinary tasks, and 10-12 for genuinely multi-part, technical, mathematical, coding, policy, safety, scientific, or long-form tasks.
4. Keep each line short and grounded in the user's request; do not use placeholders, examples, or template words.
5. Do not put explanatory second sentences in skeleton lines; put all extra criteria, caveats, and checks in the matching reason paragraph.
6. Mark at least half the lines \textasciigrave{}[HIGH]\textasciigrave{}.
7. Add diagnostic steps only when they are earned by real difficulty: ambiguity, competing approaches, missing information, verification risk, or a failed path.
8. Add importance steps when a judgment deserves extra care because it affects a main claim, recommendation, method choice, numerical interpretation, safety or risk tradeoff, paper-style conclusion, or final prioritization.
9. Also add information-density steps when the value of the response depends on interpreting user intent, weighing audience impact, selecting a framing for a conclusion, or separating a robust claim from a tempting but overconfident one.
10. For nontrivial tasks, include at least one diagnostic step, importance step, or information-density step, and usually no more than three total.
11. Do not repeat the same tag or sentence frame three times in a row.
12. In skeleton lines, name the reasoning operation, check, selection criterion, or response shape rather than reusing the assistant turn's surface wording.

\#\#\# Reason Rules

1. Write one paragraph for each skeleton line, in the same order, with exactly one blank line between paragraphs.
2. Each paragraph should be compact but informative. Use one sentence for plain setup or wrap-up; use two sentences for most \textasciigrave{}[HIGH]\textasciigrave{}, comparison, check, reflection, technical, numerical, safety, or recommendation steps.
3. Reconstruct the reasoning process as a derivational trace. Use task-level categories, ordinary terms, and any task-relevant terms needed for faithful reasoning; make the connection through criteria, checks, and transitions rather than surface restatement.
4. Build toward the response stance progressively. Earlier paragraphs should develop criteria, constraints, alternatives, and checks; the final paragraphs may integrate the conclusion direction when it follows from the trace.
5. For \textasciigrave{}[EVAL]\textasciigrave{}, \textasciigrave{}[BRCH]\textasciigrave{}, \textasciigrave{}[RFLX]\textasciigrave{}, and \textasciigrave{}[BTRK]\textasciigrave{}, name the actual uncertainty, tradeoff, evidence standard, or correction. Avoid generic claims that something is simply correct.
6. For importance-driven or information-density-driven depth, explain why the point deserves extra care and what evidence, constraint, consequence, audience need, or uncertainty boundary keeps it honest.
7. When adding depth, discuss the evidence standard, boundary condition, stakeholder need, format constraint, uncertainty, or failure mode before narrowing the response direction.
8. When the assistant turn contains a strong claim, concrete recommendation, named method, or exact wording, ground it by articulating the underlying criteria and transition so the trace reads as derivation rather than quotation.
9. For ordinary or harder tasks, several paragraphs should contain a second sentence that records a constraint, caveat, comparison, verification target, or reflective check. Do not add such a sentence when it would merely pad an obvious step.
10. Keep reasoning task-specific, but avoid broad background, extra examples, and meta-commentary about this protocol.
11. Do not number paragraphs, copy skeleton tags, or output anything outside the two required blocks.
12. The reasoning paragraphs must be between \textasciigrave{}\textless{}reason\textgreater{}\textasciigrave{} and \textasciigrave{}\textless{}/reason\textgreater{}\textasciigrave{}, not before the opening tag or after the closing tag.
13. Always write the final closing tag \textasciigrave{}\textless{}/reason\textgreater{}\textasciigrave{} after the last reasoning paragraph.
14. Close \textasciigrave{}\textless{}/skeleton\textgreater{}\textasciigrave{} immediately after the final skeleton line. Then write \textasciigrave{}\textless{}reason\textgreater{}\textasciigrave{}, the reasoning paragraphs, and finally \textasciigrave{}\textless{}/reason\textgreater{}\textasciigrave{}.
\end{tcolorbox}

\vskip 0.1in

\begin{tcolorbox}[promptstyle, title=Skeleton Generation Prompt (SS-GEN)]
You are an expert AI assistant converting an existing reasoning trace into an SSR structural skeleton.

Your job is to analyze the provided INPUT, segment it into faithful reasoning operations, and output only a concise SSR skeleton. Do not write new reasoning, conclusions, titles, or explanations.

\#\#\# Tags

Use these tags only:

- [PLAN]: understand the request, constraints, and answer shape.
- [RETR]: bring in needed facts, context, or remembered information.
- [INFR]: infer, calculate, transform, or connect task details.
- [EVAL]: check a risky assumption, calculation, source, or consistency point.
- [SUMM]: organize the constraints, checks, and response plan.
- [BRCH]: compare plausible approaches before selecting one.
- [RFLX]: pause on an important judgment and name what keeps it disciplined.
- [BTRK]: revise a concrete earlier path after a check fails.

\#\#\# Skeleton Rules

1. Segment the INPUT into coherent operations that actually appear in the trace, preserving the original logical order.
2. Output one numbered line per operation using exactly this format: \textasciigrave{}n. [TAG][HIGH/LOW] \textless{}short action sentence\textgreater{}\textasciigrave{}.
3. Choose one real tag from the tag list, then write \textasciigrave{}[HIGH]\textasciigrave{} or \textasciigrave{}[LOW]\textasciigrave{} with no space between the tag and marker.
4. Use \textasciigrave{}[HIGH]\textasciigrave{} for objectively hard, risky, or information-dense operations: ambiguity, competing approaches, calculations, verification, failed paths, central judgments, method choices, uncertainty boundaries, implementation tradeoffs, or paper-style claims.
5. Use \textasciigrave{}[LOW]\textasciigrave{} for routine setup, straightforward inference, simple retrieval, local transitions, or wrap-up.
6. Keep each action sentence short, task-specific, and operation-level. Name what the step does rather than merely restating the final result it obtains.
7. Prefer 6-12 lines when the INPUT supports that range. Merge tiny moves; split dense passages when one segment contains a distinct check, branch, reflection, or correction.
8. Do not repeat the same tag or sentence frame three times in a row.
9. Use the requested output language. If it is unspecified, match the language of the INPUT trace.
10. Every skeleton line must correspond to explicit reasoning or operations in the INPUT. Do not invent hidden steps, new facts, assumptions, or conclusions.
11. Do not correct, reinterpret, or alter the factual content of the INPUT; if the INPUT contains uncertainty or errors, reflect the operation as written.
12. Output only the formatted skeleton lines, with exactly one line break between lines. Do not include markdown fences, bullets, headings, or commentary.

Output in language: \{lang\}.

\textbf{INPUT:}
\{input\_text\}
\end{tcolorbox}

\vskip 0.1in

\paragraph{Abstract endpoint judge prompt.}

\begin{tcolorbox}[promptstyle, title=Abstract Endpoint Judge Prompt (AEJ)]
\textbf{System:}

You are a strict evaluator of abstract endpoint alignment. Given a question, a reference answer, and a candidate reasoning trace, decide whether the trace is aimed at the same abstract endpoint as the reference answer. The abstract endpoint means the same user request, response act, stance, and target deliverable or artifact type. Do not require the reasoning trace to include the reference answer's specific wording, examples, bullet points, code, numbers, entities, or detailed final-answer content; missing those details is not an error by itself. Mark inconsistent only when the trace addresses a different user request or conversation turn, takes an opposite or incompatible stance, targets a different deliverable, makes an incompatible safety/refusal decision, or is so generic that no compatible endpoint can be identified. Score 5 for clearly aligned abstract endpoint, 3 for ambiguous but plausible alignment, and 1 for a different or incompatible endpoint. Output only JSON: \{"consistent": true|false, "score": 1-5, "reason": "one sentence"\}.

\textbf{User:}

[QUESTION]
\{question\}

[REFERENCE ANSWER]
\{answer\}

[CANDIDATE REASONING]
\{trace\}
\end{tcolorbox}

\end{document}